\documentclass[11pt]{article}

\usepackage[letterpaper,margin=1in]{geometry}
\usepackage[T1]{fontenc}
\usepackage[utf8]{inputenc}
\usepackage{lmodern}
\usepackage[round,authoryear]{natbib}
\usepackage{amsmath,amssymb,amsthm,mathtools}
\usepackage{enumitem}
\usepackage{algorithm}
\usepackage{algpseudocode}
\usepackage{placeins}
\usepackage{needspace}
\usepackage[hidelinks]{hyperref}
\usepackage[nameinlink,capitalize]{cleveref}
\usepackage{microtype}
\usepackage{graphicx}
\usepackage{tikz}
\usetikzlibrary{arrows.meta,calc,positioning}
\makeatletter
\providecommand{\theHALG@line}{}
\renewcommand{\theHALG@line}{\thealgorithm.\arabic{ALG@line}}
\makeatother

\newtheorem{theorem}{Theorem}
\newtheorem{lemma}{Lemma}

\newtheorem{definition}{Definition}
\newtheorem{remark}{Remark}
\newtheorem{corollary}{Corollary}

\newcommand{\DeltaD}{\Delta_{n-1}}
\newcommand{\calG}{\mathcal G}
\newcommand{\barG}{\overline{\mathcal G}}

\newcommand{\calR}{\mathcal R}

\newcommand{\calF}{\mathcal F}

\newcommand{\one}{\mathbf 1}

\newcommand{\E}{\mathbb E}
\newcommand{\dist}{\operatorname{dist}}
\newcommand{\conv}{\operatorname{conv}}

\newcommand{\GlobPack}{\mathsf G}
\newcommand{\LocPack}{\mathsf D}
\newcommand{\MCerr}{\operatorname{MCerr}}

\usepackage{xcolor}
\title{Instance-Adaptive Online Multiclass Multicalibration}
\author{\normalsize Zhiming Huang$^{1}$ \quad Jamie Morgenstern$^{1,3}$ \quad
Aaron Roth$^{2}$ \quad Claire Jie Zhang$^{1}$\\[0.6em]
\small $^{1}$Paul G. Allen School of Computer Science and Engineering, University of Washington\\[0.3em]
\small $^{2}$Department of Computer and Information Sciences, University of Pennsylvania\\[0.3em]
\small $^{3}$Amazon}
\date{}
\hypersetup{
    pdftitle={Instance-Adaptive Online Multiclass Multicalibration},
    pdfauthor={Zhiming Huang, Jamie Morgenstern, Aaron Roth, Claire Jie Zhang}
}
\begin{document}
\maketitle

\begin{abstract}
We study online multiclass multicalibration over $T$ rounds, and give an algorithm with beyond-worst-case guarantees. With $n$ possible outcomes, forecasts lie in a simplex of dimension $d=n-1$. In the worst case, cumulative $\ell_1$ multicalibration error scales with $\widetilde\Theta(T^{(d+1)/(d+2)})$, which becomes untenable once $d$ exceeds a small constant. Our algorithm adapts to the geometry of the realized sequence: if the sequence of adversarially selected  outcome means lies in an affine subspace of dimension $s$, it obtains error $\widetilde O(n^{1/(s+2)}T^{(s+1)/(s+2)})$ with high probability. The algorithm adapts to this structure without knowing the subspace or its dimension. For \(T\ge n\ge2\), this rate is minimax optimal up to polylogarithmic factors. For a stochastic instance with a fixed conditional mean, the bound specializes to $\widetilde O(\sqrt{nT})$, giving a square-root dependence on the horizon.
\end{abstract}

\section{Introduction}
Consider forecasting which of \(n\) possible decision-relevant states will occur,
based on an observed context \(x\in\mathcal X\).  For example, predicting which token will occur next in a sequence,
where the context is the preceding token sequence.
A forecast is a probability vector within the simplex of dimension \(d=n-1\); in our example, this would correspond to a vector of probabilities of the subsequent token.
Calibration requires that, conditional on a forecast, the frequencies of the realized states match the forecast probabilities \citep{dawid1982well}.
Calibrated forecasts form a transparent interface between predictions and decision making: for every downstream decision maker (whatever their utility function or action set), their best prediction-to-action policy is to take the forecasts at face value and best respond to them as if they are correct \citep{foster1997calibrated,kleinberg2023ucalibration,noarov2025high}.
Multicalibration requires calibration simultaneously within a given collection of possibly overlapping groups of contexts in \(\mathcal X\)~\citep{hebert2018multicalibration}, and extends this guarantee to policies whose context dependence is captured by those groups.

Much of the algorithmic work on (multi)calibration focuses on forecasting probabilities of \emph{binary} outcomes, for which \(d=1\).
In high-dimensional settings, worst-case multicalibration rates deteriorate sharply, both in the batch and online cases \citep{collina2026multicalibration,collina2026sample}.
Our error measure is the natural multiclass generalization of expected calibration error (ECE), in cumulative form: within each group, sum the \(\ell_1\) deviations between each forecast and its empirical outcome frequencies, weighted by how often that forecast occurs, then take the maximum over groups. The online worst-case rate for this multicalibration error is \(\widetilde\Theta(T^{(d+1)/(d+2)})\) in the number of rounds \(T\).
This is dispiriting both because many prediction tasks we care about are fundamentally multiclass, like language prediction, and because the decision-making motivation for calibration is much stronger when forecasts can describe a richer state space.
This motivates us to ask:

\begin{quote} \textit{
Can an online algorithm make predictions that are guaranteed to be multicalibrated against adversarial sequences, yet adapt to the structure of the realized sequence to obtain substantially better rates on easier instances?}
\end{quote}

We give a positive answer: our algorithm adapts to the geometry of the sequence of outcome distributions.
At every round, the learner observes a context and commits to a distribution over predictions.
The adversary chooses an outcome distribution that may depend on the history, the current context, and the learner's distribution over predictions, but not on its freshly sampled prediction.
The learner observes the sampled outcomes, but our error guarantee adapts to the geometric structure of their conditional means, which are not directly observed.
For categorical outcomes, each mean is the vector of outcome probabilities; our model also permits observations that are themselves probability vectors.
We say the sequence of conditional means
 has affine dimension at most \(s\) if, after a possible additive shift, it lies in a linear subspace of dimension \(s\).
Whenever this is the case, our algorithm achieves multicalibration error
\[
    \widetilde O\!\left(n^{1/(s+2)}T^{(s+1)/(s+2)}\right)
\]
with high probability (\Cref{thm:uniform-affine}).

Importantly, our algorithm adapts to this structure as
the sequence is realized, without knowing the subspace or its dimension. In the special case where the conditional outcome mean is fixed, the sequence of means consists of a single point and \(s=0\).
The bound becomes \(\widetilde O(\sqrt{nT})\), giving a square root dependence on the horizon even when the ambient dimension is large.
More generally, suppose each round's conditional outcome distribution is a mixture of \(k\) fixed, possibly full-support distributions, with adversarially varying mixture weights.
The conditional means lie in the convex hull of the component means, so \(s\le k-1\), regardless of \(n\).
At \(s=n-1\), the bound recovers the worst-case horizon exponent.
The guarantee also tolerates means lying close to a low-dimensional affine set, with the approximation tradeoff specified in \Cref{rem:affine-robustness}.
For every \(0\le s\le n-1\) and \(T\ge n\ge2\), \Cref{thm:affine-lower} gives a lower bound matching this rate in expectation, up to polylogarithmic factors, on an instance whose means have affine dimension at most \(s\), with
a suitably chosen family of \(O(\log^3(T+1))\)
many groups.
Its outcome distributions and groups are fixed before play.

Our algorithm learns its effective prediction range from observed outcomes.
For each group and available forecast, the algorithm estimates the frequencies of outcomes it observes, with each round contributing in proportion to the probability assigned to that forecast and its membership
weight
in the group.
It stores estimates at several levels of precision, adding a forecast to the range when no value already stored at that level is sufficiently close.

The algorithm combines this range discovery with machinery
for ensuring  second-order swap-regret minimization under a logarithmic proper scoring rule.
We require the pairwise comparison bounds defining swap regret to hold within every group and on every interval during which both forecasts are available. The guarantee of
low swap regret for a forecast close to the estimated outcome frequencies then allows us to bound calibration error. The formal argument requires more finesse as
 the relevant comparator (prediction) may only be discovered after the rounds on which it would have helped.
We form estimates from nested stretches of observations: longer stretches support finer distinctions, while shorter ones let the analysis use comparators available at their start.
With enough data, these estimates lie close to averages of the adversary's conditional means.
Their separation at each level means that each set forms a packing; this then implies that low-dimensional geometry bounds both the number of forecasts discovered and the error accumulated during discovery.
Finally, the algorithm uses its observed calibration error to choose how finely to distinguish forecasts.

\subsection{Related work}

\paragraph{Instance-adaptive calibration.} The most closely related works are
\citet{huang2026instance} and \citet{liu2026adaptive}, who obtain instance-adaptive calibration guarantees for one-dimensional predictions.
\citet{huang2026instance} preserve the worst-case \(\widetilde O(T^{2/3})\) cumulative error bound for binary multicalibration while improving on structured instances, including a \(\widetilde O(\sqrt T)\) bound for i.i.d. outcomes in the marginal setting, without context-defined groups.
\citet{liu2026adaptive} give scalar marginal calibration bounds that improve when the outcome mean stays nearly constant or changes only a few times.
We study multiclass multicalibration, where forecasts are probability vectors and worst-case rates deteriorate with the number of outcomes.
This is not a coordinatewise extension of scalar calibration: our guarantee conditions on the entire predicted vector, and the low-dimensional structure need not align with the coordinate axes.

\paragraph{Multiclass and high-dimensional calibration.}
For online marginal \(\ell_1\) calibration, \citet{peng2025high} and \citet{fishelson2025highdim} achieve any constant  average calibration error after a number of rounds polynomial in dimension. Like us, they construct forecasts from averages of past observations. Our guarantees additionally hold within groups and adapt to the affine dimension of the conditional means. For instances with low dimensional structure, we give improved rates.

In the batch setting, \citet{gopalan2024multiclass} achieve projected smooth calibration with both sample size and running time polynomial in the number of classes at fixed accuracy.
This guarantee ensures that, for every subset of labels, the total probability assigned to that subset is close to a calibrated binary predictor for membership in the subset. This is weaker than full calibration.

\paragraph{Worst-case rates.}
Even for binary marginal calibration, \citet{qiao2021stronger} rule out a worst-case \(\widetilde O(\sqrt T)\) cumulative error bound, while \citet{dagan2025breaking} obtain an upper bound below \(T^{2/3}\).
For multicalibration, \citet{collina2026multicalibration} establish a binary \(\widetilde\Omega(T^{2/3})\) cumulative error lower bound, with a forthcoming revision extending the bound to \(\widetilde\Omega(T^{(d+1)/(d+2)})\) in ambient dimension \(d\).
Our lower bound builds on their construction while separating the affine dimension of the conditional means from the number of outcome coordinates (\Cref{thm:affine-lower}).

\paragraph{Other online calibration methods.}
Randomization makes calibration possible against arbitrary outcome sequences \citep{foster1998asymptotic}, with internal- and swap-regret reductions providing an algorithmic route \citep{foster1997calibrated,blum2007external}.
\citet{gupta2022online} give online group-conditional guarantees for mean predictions, moments, and prediction intervals.
For scalar outcomes, \citet{noarov2025high,ghuge2025improved} obtain \(\widetilde O(T^{2/3})\) cumulative ECE for finite group families, while \citet{garg2024oracle} reduce multicalibration over infinite benchmark classes to online learning.
For scalar predictions, \citet{hu2026swap} use online-learning benchmarks to control squared calibration error; \citet{fishelson2025full} identify probability-weighted squared marginal calibration error with full swap regret.
Our prediction component combines the subsequence-regret framework of \citet{lee2022online} with the second-order comparison update of \citet{gaillard2014second}.

\section{Problem and results}

\subsection{Online multicalibration}
\label{def:uniform-model}

We denote \([q]=\{1,\ldots,q\}\) for any positive integer \(q\). Fix a horizon \(T\in\mathbb N_{\ge1}\), confidence level \(\delta\in(0,1)\), and number of outcome coordinates or classes \(n\ge1\).
A multiclass prediction lies in the simplex $\DeltaD=\left\{p\in\mathbb R_{\ge0}^{n}:\sum_{i=1}^{n}p_i=1\right\}$, whose vertices are the standard basis vectors \(e_1,\ldots,e_n\). Its coordinates give the probabilities assigned to the $n$ classes.

At each round \(t\), the learner observes a context \(x_t\in\mathcal X\) and announces a probability distribution \(\pi_t\) over \(\DeltaD\). It samples and announces a prediction \(p_t\sim\pi_t\) using fresh private randomness, and then observes an outcome \(y_t\in\DeltaD\).
Categorical observations are the special case \(y_t\in\{e_1,\ldots,e_n\}\); both the upper and lower bounds below use the model allowing any \(y_t\in\DeltaD\).

We assume an adversary generates the observations $y_t$ above, in a way which can depend upon
the past, on \(x_t\), and on the announced mixture \(\pi_t\), but not on the fresh sample.
Let \((\mathcal F_t^-)_{t=1}^T\) be the filtration, where \(\mathcal F_t^-\) is generated by the information available immediately before \(p_t\) is sampled, including the past, the current context \(x_t\), the learner’s internal state, and the announced distribution \(\pi_t\).
Write
\[
    \mu_t:=\E[y_t\mid \calF_t^-] .
\]
The fresh private prediction \(p_t\) is conditionally independent of \(y_t\) given \(\calF_t^-\).

Fix a known finite, nonempty family \(\calG\) of \emph{groups}, where each \(g\in\calG\) is a function
 \(g:\mathcal X\to[0,1]\). For any  \(x_t\in \mathcal X\), the learner can evaluate \(g(x_t)\) exactly for every \(g\in\calG\).
The realized multiclass multicalibration error of an algorithm on a sequence is defined as
\[
\MCerr_{\calG}^{(1)}(T)
:=
\max_{g\in\calG}
\sum_{r\in\{p_1,\ldots,p_T\}}
\left\|
\sum_{t:p_t=r}g(x_t)(y_t-r)
\right\|_1 .
\]
For a group of total weight \(W_g=\sum_tg(x_t)>0\), its term in the maximum is \(W_g\) times its weighted multiclass expected calibration error (ECE). We augment the set of groups to contain
the constant group \(\one(x)=1\)  and denote it as \(\barG=\calG\cup\{\one\}\).

\subsection{Adaptation to affine dimension}

We now introduce the machinery necessary to define when the sequence of conditional means has low effective dimension. This roughly translates to belonging to a low dimensional linear subspace, though we allow for the subspace to be offset from the origin.

\begin{definition}[Affine dimension]
Let \(S\subseteq\mathbb R^n\) be nonempty.
Its affine dimension, denoted \(\operatorname{affdim}(S)\), is the dimension of the smallest affine subspace containing \(S\).
An affine subspace is a translate \(a+V\) of a linear subspace \(V\), and has dimension \(\dim V\).
\end{definition}

For example, a singleton has affine dimension \(0\), a nondegenerate line segment has affine dimension \(1\), and a planar set not contained in a line has affine dimension \(2\).
The simplex \(\DeltaD\) has affine dimension \(n-1\). These sets can also be translated away from the origin by any constant vector and have the same affine dimension.

The following theorem bounds our algorithm's multicalibration error in terms of the affine dimension of the conditional means.

\begin{theorem}[Uniform affine-dimensional upper bound]\label{thm:uniform-affine}
With the adaptive rule described in Section~\ref{sec:iclr-controller}, \Cref{alg:fixed-depth} takes inputs \(T, n, \delta, \calG\) and has the following guarantee against every adaptive adversary.
Let \(s=\operatorname{affdim}\{\mu_1,\ldots,\mu_T\}\), which may depend on the realized history.
With probability at least \(1-\delta\), the realized predictions satisfy
\begin{equation}
\label{eq:uniform-affine-upper}
\MCerr^{(1)}_{\calG}(T)
\le C
\left[\log\!\left(\frac{eTn|\barG|}{\delta}\right)\right]^2
n^{1/(s+2)}T^{(s+1)/(s+2)},
\end{equation}
where \(C>0\) is a universal constant.
\end{theorem}

The algorithm does not need to know \(s\) or the affine subspace containing the conditional means.
The guarantee also extends to conditional means that lie near a low-dimensional affine set, with an additional error term measuring their distance from that set.
\begin{remark}[Approximate low-dimensional structure]\label{rem:affine-robustness}
For a nonempty closed convex set \(S\subseteq\DeltaD\), write \(\rho_\infty(S):=\max_{t\le T}\inf_{u\in S}\|\mu_t-u\|_1\).
On the same event, simultaneously for every such \(S\) and integer \(s\) with \(\operatorname{affdim}(S)\le s\le n-1\),
\[
    \MCerr_{\calG}^{(1)}(T)
    \le\widetilde O\!\left(n^{1/(s+2)}T^{(s+1)/(s+2)}+T\sqrt{\rho_\infty(S)}\right).
\]
The \(\widetilde O\) hides only the squared logarithm in \eqref{eq:uniform-affine-upper}; \(s,S\) may be chosen after play.
\end{remark}

We complement this guarantee with a worst-case lower bound with matching dependence on \(n,T\), and the affine dimension of the conditional means, up to polylogarithmic factors.

\begin{theorem}[Affine-dimensional lower bound]\label{thm:affine-lower}
There are universal constants \(c,C>0\) such that, for all integers \(n\ge2\), \(T\ge n\), and \(0\le s\le n-1\), there is a stochastic instance with a family \(\calG\) of at most \(C\log^3(T+1)\) binary groups.
Its outcome means satisfy $\operatorname{affdim}\{\mu_1,\ldots,\mu_T\}\le s$,
and every possibly randomized online forecaster incurs
\[
    \E\!\left[\MCerr_{\calG}^{(1)}(T)\right]
    \ge
    c\frac{n^{1/(s+2)}T^{(s+1)/(s+2)}}{\log^2(T+1)}.
\]
The context sequence is deterministic, and the groups and outcome distributions are fixed before play.
For \(s=0\), categorical outcomes and the constant group suffice, with lower bound \(c\sqrt{nT}\).
\end{theorem}

Section~\ref{sec:lower-overview} explains the construction, which builds on \citet{collina2026multicalibration}; Appendix~\ref{app:affine-lower} gives the proof.

\section{The algorithm}

We now present the algorithm which proves Theorem~\ref{thm:uniform-affine}.
In order to naturally adapt to low-dimensional instances, the algorithm combines a swap-regret prediction rule with discovery of new candidate predictions from weighted outcome averages.
The discovery procedure maintains candidate prediction sets at several accuracy levels.
At each accuracy level, the candidate predictions are separated from each other, leading them to form a \emph{packing}.
As weighted observations accumulate, the discovery procedure may add new candidate predictions.

We first describe the algorithm with a fixed number \(1\le K\le K_{\max}:=2+\lceil\log_2T\rceil\) of accuracy levels.
Section~\ref{sec:iclr-controller} then explains how to adjust \(K\) using the observed calibration error.
The algorithm uses polynomially many exact real operations in \(T,n,|\calG|,\log(1/\delta)\), under the operation model specified in Appendix~\ref{cov:h-runtime}.

\subsection{Prediction rule}
\label{sec:prediction-component}

\paragraph{Prediction sets, smoothing, and loss.}
Let \(q_0=(1/n,\ldots,1/n)\).
Let \(\mathcal R_t\) be the set of predictions available at the start of round \(t\), initially \(\mathcal R_1=\{q_0\}\).
Section~\ref{def:algorithm-parameters} describes how this set is updated through discovery.
Based on the history and current context \(x_t\), a second-order swap-regret forecaster chooses a distribution \(\pi_t\) over \(\mathcal R_t\).
The learner samples and announces \(p_t\sim\pi_t\) using fresh private randomness.

After observing \(y_t\), the learner forms the smoothed outcome
\[
y'_t=(1-\beta)y_t+\beta q_0,
\]
where \(\beta=1/(2T^2)\),
and updates the forecaster using the logarithmic losses
\[
\ell(a,y'_t)=-\sum_{i=1}^n y'_{t,i}\log a_i,
\qquad a\in\mathcal R_t.
\]
Every candidate prediction discovered below is a weighted average of smoothed outcomes, so each coordinate is at least \(\beta/n\).
This keeps the losses bounded.
For the same prediction sequence, replacing \(y_t\) by \(y'_t\) changes the calibration error by at most \(2\beta T\).

We use logarithmic loss because weighted loss comparisons can be expressed as comparisons of KL divergences.
Pinsker's inequality relates KL divergence to squared \(\ell_1\) error without a dimension-dependent factor.
By contrast, squared loss directly controls squared \(\ell_2\) error, and passing from \(\ell_2\) to \(\ell_1\) error can introduce a factor of \(\sqrt n\).

\paragraph{The forecaster.}

We combine the subsequence-regret approach of \citet{lee2022online} with the second-order updates of \citet{gaillard2014second} to obtain the interval loss-comparison bound in \Cref{lem:interval-loss-comparison}.

Fix a learning rate \(0<\eta<1/8\), with the choice specified in \Cref{lem:interval-loss-comparison}.
For each group \(g\in\barG\) and ordered pair of distinct available predictions \(a,b\), maintain a comparison weight \(B_t^{g,a,b}\).
Initialize it to \(B_t^{g,a,b}=t\) on the first round when both values are available.
Maintain and update this weight while both predictions remain available.

After observing \(x_t\), form the directed edge weights \(W_{t,a,b}\) for \(a, b \in \mathcal R_t\):
\begin{equation}
W_{t,a,b}
=\sum_{g\in\overline{\mathcal G}}g(x_t)B_t^{g,a,b},
\qquad a\ne b,
\qquad W_{t,a,a}=0.
\end{equation}
Choose the unique probability vector satisfying the stationarity equations
\begin{equation}
\pi_{t,b}\sum_{c\ne b}W_{t,b,c}
=\sum_{a\ne b}\pi_{t,a}W_{t,a,b}
\quad(b\in\mathcal R_t),
\qquad
\sum_{a\in\mathcal R_t}\pi_{t,a}=1.
\label{eq:comparison-matrix}
\end{equation}
With one available prediction, give it probability one.

After sampling \(p_t \sim \pi_t\) and observing \(y_t\), set \(c_t^{g,a,b}=g(x_t)\pi_{t,a}(\ell(a,y'_t)-\ell(b,y'_t))/\log(n/\beta)\) if both predictions are available, and set it to zero otherwise.
Update the comparison state by
\begin{equation}
B_{t+1}^{g,a,b}
=1+\left(1+\eta c_t^{g,a,b}\right)B_t^{g,a,b}.
\label{eq:comparison-potential-update}
\end{equation}
The proof of the interval guarantee, including existence and uniqueness of the prediction distribution, is in Appendix~\ref{subsec:interval-loss-comparison-proof}.

\subsection{Discovering candidate predictions}
\label{def:algorithm-parameters}

The discovery procedure forms weighted averages of smoothed outcomes and stores selected averages as candidate predictions.
We first describe how these averages are formed, then explain when they are added to the candidate prediction sets.

\paragraph{Weighted observations and blocks.}
For each group \(g\in\barG\) and available prediction \(a\in\mathcal R_t\), define the round-\(t\) weight $\alpha_t^{g,a}=g(x_t)\pi_{t,a}$,
where \(\pi_{t,a}\) is the probability assigned to \(a\), and is zero when \(a\notin\mathcal R_t\).
We call \(a\) the source prediction and write \(\alpha_t\) when \(g,a\) are clear from context.
For each pair \((g,a)\), we collect the smoothed outcomes \(y'_t\), with weights \(\alpha_t^{g,a}\), into \emph{blocks} of rounds.
A nonempty block \(E\) spans the interval from its first to its last original positive-weight observation round, including any zero-weight rounds between them.
Blocks belonging to different pairs are distinct even when their intervals coincide.
For a block \(E\) associated with \((g,a)\), define its total weight and weighted outcome sum by
\[
    N_E:=\sum_{t\in E}\alpha_t^{g,a},
    \qquad
    V_E:=\sum_{t\in E}\alpha_t^{g,a}y'_t.
\]
Both sums are zero for an empty block.
When \(N_E>0\), the block's weighted mean is \(z_E=V_E/N_E\).

\paragraph{Combining observations into blocks.}
To form weighted averages at different accuracies, we maintain blocks for each pair \((g,a)\) at \(K\) levels, indexed by \(j=0,\ldots,K-1\) from coarsest to finest.
Let \(L=O(\log(en|\barG|(T+1)/\delta))\) be the logarithmic factor in the concentration bounds for weighted outcome averages.
Its value is fixed before play as specified in Appendix~\ref{app:universal-constants}.
At level \(j\), the separation threshold is \(w_j=2^{1-j}\), and the completion threshold is \(m_j=1024nL/w_j^2\).
Thus finer levels wait for more weight before completing a block.

For each pair \((g,a)\) and level \(j\), we keep at most one \emph{open block}, which is still collecting weight.
We close it as soon as its weight reaches or exceeds the completion threshold \(m_j\), including the threshold-crossing input in full.
A closed block is called \emph{completed} and receives no further inputs.

At level zero, the open block directly takes weight \(\alpha_t\) and smoothed outcomes \(y'_t\).
At level \(j\ge1\), the open block receives only the summaries \((N_E,V_E)\) of completed blocks at level \(j-1\), adding both components to its own summary.
Once a block closes, the next input starts a new block \(E'\) with \(N_{E'}\) and \(V_{E'}\) initialized at zero, as shown in \Cref{fig:block-hierarchy}.

The thresholds are chosen so that exactly four completed blocks, called \emph{children}, form a completed block at the next level, their \emph{parent}.
Every completed level-\(j\) block has weight in \([m_j,2m_j)\).
Appendix~\ref{sec:discovery-hierarchy} proves both properties.
A parent's observation interval includes the original rounds of its children, so it can begin before the parent's accumulator is created.
Blocks at different levels remain distinct even when their observation intervals coincide.

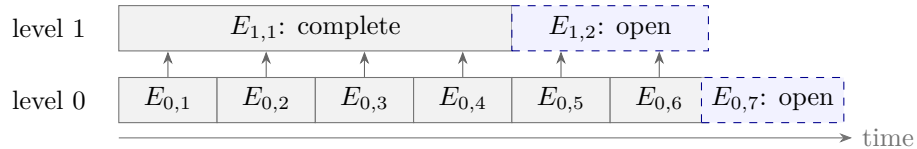
\begin{figure}[!htbp]
\centering
\begin{tikzpicture}[x=1cm,y=1cm,>=Stealth,font=\small,
    done/.style={draw=black!55,fill=black!5,minimum height=0.48cm,align=center},
    unfinished/.style={draw=blue!55!black,fill=blue!5,dashed,minimum height=0.48cm,align=center}]
\node[anchor=east] at (0.7,1.2) {level \(1\)};
\node[anchor=east] at (0.7,0.25) {level \(0\)};
\node[done,minimum width=5.2cm] (f1) at (3.6,1.2) {\(E_{1,1}\): complete};
\node[unfinished,minimum width=2.6cm] (f2) at (7.5,1.2) {\(E_{1,2}\): open};
\foreach \i in {1,...,6} {
    \pgfmathsetmacro{\xx}{1.65+1.3*(\i-1)}
    \node[done,minimum width=1.3cm] (e\i) at (\xx,0.25) {\(E_{0,\i}\)};
    \ifnum\i<5
        \draw[->,black!60] (e\i.north) -- (e\i.north |- f1.south);
    \else
        \draw[->,black!60] (e\i.north) -- (e\i.north |- f2.south);
    \fi
}
\node[unfinished,minimum width=1.7cm] at (9.65,0.25) {\(E_{0,7}\): open};
\draw[->,black!55] (1,-0.25) -- (10.7,-0.25) node[right] {time};
\end{tikzpicture}
\caption{Blocks for one \((g,a)\) pair, where \(E_{j,r}\) denotes the \(r\)-th block at level \(j\) in chronological order.
Completed blocks \(E_{0,1},\ldots,E_{0,4}\) form \(E_{1,1}\).
The open block \(E_{1,2}\) has received the summaries of \(E_{0,5},E_{0,6}\).
The block \(E_{0,7}\) is still collecting observations and has not been forwarded.
Arrows add both weight and outcome sum.
}
\label{fig:block-hierarchy}
\end{figure}

\paragraph{Candidate predictions and insertion.}
For each accuracy level \(j=0,\ldots,K-1\), let \(P_j\) denote the set of candidate predictions, shared across all \((g,a)\) pairs.
At level \(0\), we always have the fixed set \(P_0=\{q_0\}\).
Each positive-level set \(P_j\) is initially empty.
Distinct candidate predictions in \(P_j\) are kept more than \(w_j\) apart in Hellinger distance,
\[
H(p,q)=\|\sqrt p-\sqrt q\|_2,
\]
with square roots taken coordinatewise. 
The following insertion rule maintains this separation.

When a block \(E\) at a positive level \(j\) is completed, compare its mean \(z_E\) with the values in \(P_j\).
If some value in \(P_j\) is within Hellinger distance \(w_j\) of \(z_E\), no insertion is made.
Otherwise, process levels \(k=1,\ldots,j\) in increasing order, inserting \(z_E\) into \(P_k\) if \(H(b,z_E)>w_k\) for every \(b\in P_k\).
Every completed block below level \(K-1\) is forwarded to the next level, whether or not a candidate prediction was inserted.
If this completes the next-level block, process that block recursively on the same round, using the updated candidate sets.

These candidate predictions have every coordinate at least \(\beta/n\), so their KL divergence is bounded above and below by squared Hellinger distance up to a logarithmic factor (\Cref{lem:kl-distance-bounds}).

\paragraph{Updating available predictions.}
We write \(P_j(t)\) for the contents of \(P_j\) at the start of round \(t\), so the available prediction set is
\[
\mathcal R_t=P_0\cup\bigcup_{j=1}^{K-1}P_j(t).
\]
In what follows, we write \(P_j\) for the candidate prediction set at the time under consideration, making the time explicit only when needed.
After all discovery updates on round \(t\), the resulting sets are \(P_j(t+1)\).
Newly discovered candidate predictions receive prediction probability and observations only from round \(t+1\) onward.
Since the candidate prediction sets only grow, \(\mathcal R_t\) grows monotonically while \(K\) remains fixed.
Algorithm~\ref{alg:fixed-depth} specifies these operations.

\begin{algorithm}[h]
\small
\caption{Learning predictions with \(K\) accuracy levels}
\label{alg:fixed-depth}
\begin{algorithmic}
\Require \(T,n,\delta,\calG\), number of levels \(K\in\{1,\ldots,K_{\max}\}\).
\State Initialize \(P_0=\{q_0\}\), \(P_j=\varnothing\) for \(1\le j\le K-1\), empty blocks, and the prediction rule of Section~\ref{sec:prediction-component}.
\State Order groups as supplied, appending \(\one\) if absent. Order candidate predictions by first discovery, starting with \(q_0\), and use this order for ties. Repeated values keep their original position.
\For{\(t=1,\ldots,T\)}
    \State Observe \(x_t\) and set \(\mathcal R_t\gets P_0\cup\bigcup_{j=1}^{K-1}P_j\).
    \State Use the prediction rule to choose \(\pi_t\) on \(\mathcal R_t\); sample and announce \(p_t\sim\pi_t\).
    \State Observe \(y_t\), form \(y'_t\), and update the swap-regret forecaster using the losses \(\ell(a,y'_t)\) for \(a\in\mathcal R_t\).
    \ForAll{\(g\in\barG,\ a\in\mathcal R_t\) with \(\alpha_t^{g,a}=g(x_t)\pi_{t,a}>0\)}
        \State Start a zero-level block \(E\) with \((N_E,V_E)=(0,0)\) if absent.
        \State \(N_E\gets N_E+\alpha_t^{g,a}\), \(V_E\gets V_E+\alpha_t^{g,a}y'_t\).
    \EndFor
    \While{an open level-zero block has \(N_E\ge m_0\)}
        \State Choose the first such block in source-prediction order, then group order.
        \State Call \(\textsc{ProcessThresholdBlock}(E,0,g,a)\).
    \EndWhile
\EndFor
\Procedure{ProcessThresholdBlock}{$E,j,g,a$}
    \State Remove \(E\) from the open blocks, mark it completed, and set \(z_E\gets V_E/N_E\).
    \If{\(j\ge1\) and \(H(b,z_E)>w_j\) for every \(b\in P_j\)}
        \For{\(k=1,\ldots,j\)}
            \State If \(H(b,z_E)>w_k\) for every \(b\in P_k\), insert \(z_E\) into \(P_k\).
        \EndFor
    \EndIf
    \If{\(j<K-1\)}
        \State Get the open level-\((j+1)\) block \(E'\) of \((g,a)\), starting it at \((0,0)\) if absent.
        \State \(N_{E'}\gets N_{E'}+N_E\), \(V_{E'}\gets V_{E'}+V_E\).
        \State If \(N_{E'}\ge m_{j+1}\), call \(\textsc{ProcessThresholdBlock}({E'},j+1,g,a)\).
    \EndIf
    \State Discard the completed block's stored summary.
\EndProcedure
\end{algorithmic}
\end{algorithm}

\FloatBarrier

\section{Bounding the calibration error}

We bound the calibration error of Algorithm~\ref{alg:fixed-depth} in three steps.
First, we bound a block's weighted error when a prediction close to its empirical mean is available throughout the block.
Next, we partition the observations into residual pieces to account for predictions discovered later.
Finally, we use the low-dimensional geometry of the conditional means to bound the resulting discovery cost.

The rule in Section~\ref{sec:iclr-controller} may decrease \(K\), clearing blocks but retaining candidate prediction sets.
Call a maximal stretch with the same \(K\) a \emph{depth-\(K\) interval}.

\subsection{An accurate available comparator controls error}
The prediction rule bounds the loss advantage of switching from a source prediction \(a\) to a comparator \(q\).
Its second-order guarantee depends on the squared loss differences on the same rounds.
For logarithmic loss, these quadratic terms are themselves controlled by the divergences from the weighted outcome mean to \(a\) and \(q\).
This is what turns regret into a multiplicative comparison of their prediction errors.

\begin{lemma}[Interval loss-comparison bound]\label{lem:interval-loss-comparison}
Let \(\kappa=2+\log(n/\beta)\).
With learning rate \(\eta=\log(n/\beta)/(8\kappa)\), the prediction rule in Section~\ref{sec:prediction-component} satisfies the following on every realization.
For every \(g\in\barG\), pair \(a,q\), and contiguous interval \(I\subseteq[T]\) on which both predictions are available, let \(\alpha_t=g(x_t)\pi_{t,a}\) and \(X_t=\ell(a,y'_t)-\ell(q,y'_t)\).
Then
\begin{equation}
    \sum_{t\in I}\alpha_tX_t
    \le\frac1{8\kappa}\sum_{t\in I}\alpha_tX_t^2+8\kappa L.
\label{eq:second-order-comparison}
\end{equation}
\end{lemma}

With sums over \(I\), write \(N=\sum_t\alpha_t>0\) and \(z=N^{-1}\sum_t\alpha_ty'_t\).
The KL divergence is \(D(z\Vert a)=\sum_i z_i\log(z_i/a_i)\).
Substituting logarithmic loss gives the identity
\(
    \sum_t\alpha_tX_t=N[D(z\Vert a)-D(z\Vert q)].
\)
The quadratic terms and thus the right-hand side of the inequality can be bounded using these same divergences:
\(\sum_t\alpha_tX_t^2\le2\kappa N[D(z\Vert a)+D(z\Vert q)]\).
Substitution into \eqref{eq:second-order-comparison} lets us absorb the source divergence into the left-hand side, yielding the multiplicative comparison
\(
    D(z\Vert a)\le3D(z\Vert q)+16\kappa L/N.
\)

An accurate comparator therefore forces the source prediction to be accurate, with an additive divergence cost of order \(\kappa L/N\).
The following lemma converts this divergence comparison into bounds on the distance between the source prediction \(a\) and the weighted outcome mean \(z\), when the comparator \(q\) is available throughout the interval.

\begin{lemma}[Available-comparator bound]
\label{lem:local-comparison}
The two bounds below hold on every realization. For every group and contiguous round interval on which \(a,q\) are both available, with \(N,z\) defined above,
\begin{equation}
    H(z,a)\le\sqrt{3\kappa}H(z,q)+4\sqrt{\kappa L/N}.
\label{eq:iclr-local-hellinger}
\end{equation}
If \(N\ge m_j\) and \(H(z,q)\le2w_j\), then
\begin{equation}
    \|z-a\|_1\le6\sqrt{\kappa}\,w_j.
\label{eq:iclr-block-residual}
\end{equation}
\end{lemma}

\subsection{Accounting for comparators discovered later}
\label{def:discovery-counts}
\label{def:boundary-analysis}

The preceding argument controls a block whenever a prediction close to its empirical mean is available throughout the block. The main complication is that such a comparator may be discovered only when the block is completed.

Let \(I\) consist of all rounds from the start of a depth-\(K\) interval through any round within it, after that round's discovery updates.
We first assign a comparator to each completed block.
For a positive-level block \(E\) at level \(j\), inspect \(P_j\) immediately before processing \(E\)'s insertion tests.
If a candidate prediction in \(P_j\) lies within Hellinger distance \(w_j\) of \(z_E\), let \(q_E\) be the earliest-discovered such value.
Otherwise let \(q_E=z_E\), the value inserted into \(P_j\) at this closure.
At level zero, set \(q_E=q_0\).

For each pair \((g,a)\), view the blocks as trees, joining each block to the parent that receives its summary.
At the end of \(I\), start from the maximal blocks, meaning those with no parent yet.
Write \(\min E\) for the first original round in \(E\).
Select a partition by the following recursion:
\begin{itemize}
    \item If \(E\) is completed and \(q_E\in\mathcal R_{\min E}\), select \(E\) and do not visit its children.
    Since the available prediction set only grows while \(K\) is fixed, \(q_E\) was available throughout \(E\).
    Otherwise replace \(E\) by its children and apply the same rule to each.
    \item If \(E\) is unfinished and at a positive level, apply the rule to its completed children.
    Select every unfinished level-zero block with positive weight as one piece.
\end{itemize}
The recursion terminates because \(q_0\) is available throughout every level-zero block.
The selected blocks form a partition \(\mathcal P_{g,a}(I)\) of the pair's weighted observations.
We call them \emph{residual pieces}; \Cref{fig:residual-piece-example} illustrates the selection.
This partition is used only in the analysis, so the algorithm need not store the block trees.

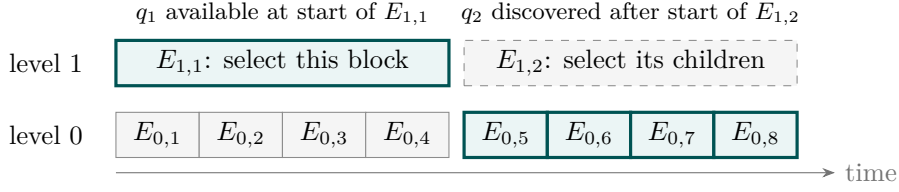
\begin{figure}[!htbp]
\centering
\begin{tikzpicture}[x=1cm,y=1cm,>=Stealth,font=\small,
    kept/.style={draw=teal!65!black,very thick,fill=teal!8,minimum height=0.48cm,align=center},
    unused/.style={draw=black!45,fill=black!4,minimum height=0.48cm,align=center}]
\node[anchor=east] at (0.7,1.35) {level \(1\)};
\node[anchor=east] at (0.7,0.4) {level \(0\)};
\node[kept,minimum width=4.4cm] (f1) at (3.2,1.35) {\(E_{1,1}\): select this block};
\node[unused,dashed,minimum width=4.4cm] (f2) at (7.8,1.35) {\(E_{1,2}\): select its children};
\node[font=\footnotesize] at (3.2,2.0) {\(q_1\) available at start of \(E_{1,1}\)};
\node[font=\footnotesize] at (7.8,2.0) {\(q_2\) discovered after start of \(E_{1,2}\)};
\foreach \i in {1,...,4} {
    \pgfmathsetmacro{\xx}{1.55+1.1*(\i-1)}
    \node[unused,minimum width=1.1cm] at (\xx,0.4) {\(E_{0,\i}\)};
    \pgfmathsetmacro{\xx}{6.15+1.1*(\i-1)}
    \pgfmathtruncatemacro{\r}{\i+4}
    \node[kept,minimum width=1.1cm] at (\xx,0.4) {\(E_{0,\r}\)};
}
\draw[->,black!55] (1,-0.1) -- (10.5,-0.1) node[right] {time};
\end{tikzpicture}
\caption{Selecting pieces for one group--source pair. Both parents are complete, with comparators \(q_i=q_{E_{1,i}}\) satisfying \(H(q_i,z_{E_{1,i}})\le w_1\).
Select \(E_{1,1}\) because \(q_1\) was available throughout it; replace \(E_{1,2}\) by its children because \(q_2\) was unavailable at its start.
The blocks with thick outlines partition the observations. This selection is used only in the proof.
}
\label{fig:residual-piece-example}
\end{figure}
For each pair \((g,a)\), the triangle inequality gives
\[
\left\|\sum_{t\in I}\alpha_t^{g,a}(y'_t-a)\right\|_1
\le \sum_E N_E\|z_E-a\|_1,
\]
where \(E\) ranges over \(\mathcal P_{g,a}(I)\).
The sampling and smoothing corrections given below then convert this into a bound on the realized calibration error.
For every completed level-\(j\) piece, the available-comparator bound gives
\begin{equation}\label{eq:errorbounds}
    N_E\|z_E-a\|_1
    \le 6\sqrt{\kappa}\,N_Ew_j.
\end{equation}
For each fixed group \(g\), we sum \eqref{eq:errorbounds} over all completed
residual pieces across source predictions.
We bound the contributions from pieces at the finest level \(K-1\)
and from coarser pieces separately.

The pieces at the finest level \(K-1\) contribute at most \(6\sqrt{\kappa} |I| w_{K-1}\) due to \(\sum_E N_E \leq |I|\).
We next control the contribution of coarser pieces by bounding how many are selected.
Call a completed positive-level block \(E\) a \emph{boundary block} if \(q_E\notin\mathcal R_{\min E}\): its comparator was unavailable on its first original observation round.
A selected completed piece below the finest level must have a boundary or unfinished parent.
We therefore track the number of candidate predictions and the number of
boundary blocks for each group, source prediction, and level.
Let
\[
M_I:=\left|P_0\cup\bigcup_{j=1}^{K-1}P_j(\max I+1)\right|,
\]
where \(K\) is its value on \(I\).
This counts distinct candidate predictions across these levels, including all insertions on the final round of \(I\), before any decrease of \(K\).
Define
\[
B_I:=\max\left\{
1,\,
\max_{\substack{g,a\\1\le j\le K-1}}
\#\left\{
E:\ E \text{ is a level-\(j\) boundary block for }(g,a)
\text{ in }I
\right\}
\right\},
\]
where \(g\in\overline{\mathcal G}\) and \(a\) ranges over the candidate
predictions counted by \(M_I\).

For a fixed source and group, each coarser level has at most \(B_I\) completed boundary parents, with four children each, and one unfinished parent, with at most three completed children.
Summing over the at most \(M_I\) sources gives at most \((4B_I+3)M_I\) selected completed pieces at that level.
Each has weight less than \(2m_j\), so \eqref{eq:errorbounds} and the geometric sum of \(1/w_j\) give a total coarse-level cost of \(O(\sqrt{\kappa}\,nLM_IB_I/w_{K-1})\).
There is also at most one unfinished level-zero piece per source, each of weight below \(m_0\) and cost at most twice its weight.
This cost is absorbed into the same bound.
\Cref{lem:residual-charging,lem:piece-count-residual} give the formal counts and constants.

We also need to account for the sampling and smoothing cost. On the common event \(\mathcal E_{\rm global}\) of \Cref{lem:common-event}, which has probability at least \(1-\delta\), the sampling correction on each interval \(I\) is at most \(4\sqrt{nLM_IT}\).
Smoothing adds at most \(2\beta|I|\). Write \(\MCerr_{\calG}^{(1)}(I)\) for the realized multicalibration error restricted to rounds in \(I\).
Combining these corrections with the residual-piece bounds gives the following corollary.

\begin{corollary}[Calibration error from discovery counts]
\label{cor:discovery-residual}
There is a universal constant \(C\) such that, on \(\mathcal E_{\rm global}\), every interval \(I\) defined above, with \(K\ge2\) and terminal resolution \(w=w_{K-1}\), satisfies
\[
    \MCerr_{\calG}^{(1)}(I)
    \le
    C\sqrt{\kappa}
    \left(
        |I|w+\frac{nLM_IB_I}{w}
    \right)
    +4\sqrt{nLM_IT}+2\beta|I|.
\]
\end{corollary}
The term \(|I|w\) is the ordinary approximation cost at the finest level \(K-1\). The term \(nLM_IB_I/w\) is the price of discovering useful prediction values online. 
The formal residual decomposition and piece-counting argument appear in Appendices~\ref{subsec:partition-piece-counts}--\ref{subsec:general-packing-proof}.

\subsection{Low-dimensional means limit discovery}
Next, we bound \(M_I\) and \(B_I\) using the low-dimensional geometry of the conditional means. Suppose the means lie in, or sufficiently near, a convex set \(S\) of affine dimension \(s\).
By \Cref{lem:common-event}, on the event \(\mathcal E_{\rm global}\), every completed level-\(j\) block satisfies
\[
H(z_E,\bar\mu'_E)\le w_j/16,
\qquad
\bar\mu'_E:=\frac{1}{N_E}\sum_{t\in E}\alpha_t^{g,a}\bigl((1-\beta)\mu_t+\beta q_0\bigr).
\]
The vector \(\bar\mu'_E\) is the weighted average of the smoothed conditional
means \(\mu'_t=(1-\beta)\mu_t+\beta q_0\), using the same weights as \(z_E\).
If the conditional means lie in the convex set \(S\), then
\(\bar\mu'_E\in S_\beta=(1-\beta)S+\beta q_0\).
The concentration bound above therefore places the mean \(z_E\) of every
completed level-\(j\) block within \(w_j/16\) of \(S_\beta\). Consequently, every candidate prediction in \(P_j\), \(j\ge1\), lies within \(w_j/16\) of \(S_\beta\).
For approximate structure, \(H(p,q)^2\le\|p-q\|_1\) adds at most \(\sqrt{\rho_\infty(S)}\) to this distance.

To bound \(M_I\), we use the \(w_j\)-separation of candidate predictions in \(P_j\).
A packing argument near \(S_\beta\) bounds their number at each level.
Summing over levels gives the bound on \(M_I\).

To bound \(B_I\), we exploit the chronological order of boundary blocks.
Fix a group, source prediction \(a\), and positive level \(j\).
Near \(a\), the \(w_j\)-separation of boundary comparators gives a local packing bound.
Far from \(a\), every earlier comparator is available throughout each later block.
The Hellinger bound \eqref{eq:iclr-local-hellinger} then gives separation proportional to the later comparator's distance from \(a\), up to \(\kappa\)-dependent factors.
Within distance ranges with doubling radii, this gives a local packing bound independent of the radius.
Summing over the near region and these ranges yields a bound on \(B_I\) with no explicit dependence on \(w\).

The following lemma states the resulting discovery counts.

\begin{lemma}[Affine prediction and boundary-block counts]
\label{lem:iclr-affine-counts}
There are universal constants \(a,b\ge1\) such that, on \(\mathcal E_{\rm global}\), for every nonempty closed convex set \(S\subseteq\DeltaD\), every integer \(s\) with \(\operatorname{affdim}(S)\le s\le n-1\), every \(2\le K\le K_{\max}\) with \(w=w_{K-1}\ge64\sqrt{\rho_\infty(S)}\), and every interval \(I\) at depth \(K\) as defined in Section~\ref{def:discovery-counts},
\[
    M_I\le K_{\max}(a\sqrt{\kappa}/w)^s,
    \qquad
    B_I\le K_{\max}(b\kappa)^s.
\]
\end{lemma}
Appendix~\ref{subsec:affine-hellinger-geometry} establishes the global and local Hellinger packing bounds used to prove \Cref{lem:iclr-affine-counts}.
Combining \Cref{cor:discovery-residual} and \Cref{lem:iclr-affine-counts} gives a calibration bound in terms of the resolution \(w\) for a depth-\(K\) interval.
We next show how to adjust this depth \(K\) using the observed calibration error, without knowing \(S\) or \(s\).

\FloatBarrier
\section{Choosing the resolution}
\label{sec:iclr-controller}

Start at the finest permitted resolution, \(K=K_{\max}\). Define a calibration error budget for each resolution  \(\Gamma_K=C_{\rm test}\sqrt{\kappa}\,T w_{K-1},\) where \(C_{\rm test}\) is the fixed universal constant in Appendix~\ref{app:universal-constants}.
Let \(I\) be the rounds spent at the current \(K\) so far.
After each round, including all discovery updates, compute \(\MCerr_{\calG}^{(1)}(I)\) from the realized predictions and outcomes in \(I\).
If \(\MCerr_{\calG}^{(1)}(I)>\Gamma_K\), \(K>1\), and \(t<T\), end \(I\), decrease \(K\) by one, and begin a new interval with empty block accumulators and calibration error reset to zero.
Retain all candidate prediction sets, but only levels \(0,\ldots,K-1\) for the new value of \(K\) contribute to the next round's available prediction set.
A candidate prediction becomes unavailable only if it belongs to none of these levels, even if it remains in a candidate prediction set at a higher level.
Comparison states for predictions that remain available are preserved.
Appendix~\ref{app:adaptive-procedure} gives the pseudocode.

\textbf{Why the test finds a suitable resolution.}
We sketch the proof of \Cref{thm:uniform-affine}.
On \(\mathcal E_{\rm global}\), we identify a comparison resolution at which every prefix has error below its budget.
If the algorithm reaches this resolution, it cannot move to a coarser one.
We choose this resolution by balancing the finest-level and discovery costs.
Ignoring coefficients, this gives
\[
    Tw\ \asymp\ \frac{n}{w^{s+1}},
    \qquad w\ \asymp\ (n/T)^{1/(s+2)}.
\]
Using \Cref{cor:discovery-residual,lem:iclr-affine-counts}, Appendix~\ref{app:iclr-affine-upper-proof} verifies that the calibration error accumulated since entering this resolution, including sampling and smoothing corrections, stays below the budget after every round.

Each visited interval has error at most \(\Gamma_K+2\).
The budgets double as the resolution becomes coarser, and the algorithm never passes the comparison resolution.
Their sum is therefore at most twice the budget at that resolution, with an additional \(2K_{\max}\) from overshoots.
The triangle inequality then bounds the total calibration error.
Appendix~\ref{app:iclr-affine-upper-proof} gives the full proof yielding \Cref{thm:uniform-affine} and \Cref{rem:affine-robustness}.    

\section{Matching lower bound}
\label{sec:lower-overview}

The lower bound must distinguish two sources of difficulty: the dimension \(s\) of the conditional means and the number \(n\) of outcome coordinates.
We adapt the construction of \citet{collina2026multicalibration} to vary these quantities separately.

\paragraph{Low-dimensional means with noise in every coordinate.}
Consider first \(s\ge1\) and \(T\ge\max\{n,2^{s+2}\}\).
A deterministic context sequence cycles through \([m]^s\), where \(m\ge2\) will be chosen below.
Each context \(x\) is mapped to a probability vector \(q(x)\).
We choose \(q(x)\) so that each of its first \(s\) coordinates takes one of \(m\) equally spaced values, forming an \(s\)-dimensional grid with \(m^s\) points. Its \((s+1)\)-st coordinate is \(1-\sum_{j=1}^s q_j(x)\), and all remaining coordinates are zero.
Thus the vectors \(q(x)\) lie in an affine subspace of dimension \(s\).
On round \(t\), draw \(J_t\) independently and uniformly from \([n]\), independently of the forecaster's randomness, and define the outcome vector and its conditional mean by
\[
y_t=\frac{q(x_t)+e_{J_t}}{2},
\qquad
\mu_t=\frac{q(x_t)+q_0}{2}.
\]
Here \(q_0=(1/n,\ldots,1/n)\).
Averaging the vectors \(q(x_t)\) with the fixed vector \(q_0\) does not increase affine dimension, so the conditional means \(\mu_t\) have affine dimension at most \(s\).
Furthermore, our construction of \(q(x)\) ensures that conditional means associated with distinct contexts are separated by at least \(\Delta:=1/[8(s+1)(m-1)]\) in \(\ell_1\) distance.
Although the means have affine dimension at most \(s\), the random deviation \(y_t-\mu_t=(e_{J_t}-q_0)/2\) affects all \(n\) coordinates.

\paragraph{Either prediction distance or accumulated noise forces error.}
Fix any forecaster. We consider two cases according to whether
\(\E\sum_{t=1}^T\|p_t-\mu_t\|_1\) is large or small.
Using the Walsh approximation of threshold tests from \citet{collina2026multicalibration}, we construct a fixed family \(\calG\) of context groups for which
\[
\E[\MCerr_{\calG}^{(1)}(T)]
\gtrsim
\frac{\E\sum_{t=1}^T\|p_t-\mu_t\|_1}{\log(T+1)}.
\]
Thus, using \(\Delta=1/[8(s+1)(m-1)]\), if \(\E\sum_t\|p_t-\mu_t\|_1\ge T\Delta/8\), the expected calibration error is at least a constant multiple of \(T/[(s+1)m\log(T+1)]\).

Otherwise, Markov's inequality implies that the expected number of rounds with \(\|p_t-\mu_t\|_1<\Delta/4\) is at least \(T/2\).
Since distinct conditional means are at least \(\Delta\) apart, each prediction value can be this close to at most one conditional mean.
Since each context occurs at most \(\lceil T/m^s\rceil\) times, these close predictions must use many distinct prediction values. 
For each realized prediction value \(v\), let \(n_v=|\{t:p_t=v\}|\) be its number of occurrences.
Using these bounds on the number of rounds with \(\|p_t-\mu_t\|_1<\Delta/4\) and applying the adaptive-noise inequality of \citet{collina2026multicalibration} to disjoint coordinate pairs gives
\[
\E\sum_v\left\|\sum_{t:p_t=v}(y_t-\mu_t)\right\|_1
\gtrsim
\frac{\E\sum_v\min\{n_v,\sqrt{n n_v}\}}{\log(T+1)}
\gtrsim
\frac{\min\{T,\sqrt{nT m^s}\}}{\log(T+1)}.
\]
This inequality allows predictions to depend on past outcomes.
By the triangle inequality and the constant group, the left-hand side is at most the expected calibration error plus the expected total prediction distance.
Combining this with the distance comparison above gives
\[
\E\bigl[\MCerr_{\calG}^{(1)}(T)\bigr]
\gtrsim
\frac{\min\{T,\sqrt{nT m^s}\}}{\log^2(T+1)}.
\]

\paragraph{Balancing the two bounds.}
In the range under consideration, \(s+1=O(\log(T+1))\), so the first case gives at least a constant multiple of \(T/[m\log^2(T+1)]\).
The two cases therefore balance at
\[
\frac{T}{m}\asymp\sqrt{nT m^s},
\qquad\text{or equivalently}\qquad
m\asymp(T/n)^{1/(s+2)}.
\]
Rounding \(m\) appropriately gives the lower bound \(n^{1/(s+2)}T^{(s+1)/(s+2)}/\log^2(T+1)\) in \Cref{thm:affine-lower}, using \(O(\log^3(T+1))\) groups.
For \(s=0\), independent uniformly chosen simplex vertices give a lower bound of order \(\sqrt{nT}\). Appendix~\ref{app:affine-lower} gives the full proof.

\section*{Acknowledgments}
\addcontentsline{toc}{section}{Acknowledgments}
AR is supported in part by the Simons Collaboration on Algorithmic Fairness, NSF grant IIS-2551608, and the NSF ENCoRE TRIPODS institute. ZH is supported in part by NSERC. 

Generative AI tools were used to explore proof approaches, draft and revise mathematical arguments and exposition, check manuscript consistency, and assist with literature searches and bibliographic verification.
The authors reviewed and revised the AI-assisted material and take responsibility for the final content, including the mathematical claims and attribution.

\bibliographystyle{plainnat}
\begingroup
\providecommand{\doi}[1]{\url{https://doi.org/#1}}
\raggedright
\bibliography{references_arxiv}
\endgroup

\appendix

\section{Block weights and discovery counts}
\label{sec:discovery-hierarchy}

This appendix proves the block-weight and discovery-count bounds for Algorithm~\ref{alg:fixed-depth}.
We use the block construction and thresholds from Section~\ref{def:algorithm-parameters}.

\paragraph{Block weights.}
Fix a group--source pair \((g,a)\).
We prove by induction that every completed level-\(j\) block has weight in \([m_j,4^j(m_0+1))\).
At level zero, this follows because the block closes when its weight first reaches \(m_0\), and each input has weight at most one.
For \(j\ge1\), the induction hypothesis implies that three children have total weight less than \(3\cdot4^{j-1}(m_0+1)\le 4^jm_0 = m_j\), since \(m_0\ge3\).
Four children have total weight at least \(4m_{j-1}=m_j\), so the block closes upon receiving its fourth child.
Its weight is therefore less than \(4\cdot4^{j-1}(m_0+1)=4^j(m_0+1)\), completing the induction.
Consequently, every completed positive-level block has exactly four children. Since \(4^j(m_0+1)=m_j(1+1/m_0)\le2m_j\), every completed level-\(j\) block has weight in \([m_j,2m_j)\).

\paragraph{Bounding the number of candidate predictions.}
Distinct candidate predictions are indexed by their first discovery.
For each fixed group \(g\) and scale \(j\), completed blocks across sources have total weight at most
\[
\sum_{t=1}^T\sum_{a\in\calR_t}g(x_t)\pi_{t,a}
=\sum_{t=1}^T g(x_t)\le T.
\]
Indeed, at a fixed scale each weighted observation belongs to at most one completed block of its source, and observations are not replayed across depth intervals.
Each completed scale-\(j\) block has weight at least \(m_j\), so there are at most \(T/m_j\) such blocks per group.
Every first discovery other than \(q_0\) is caused by a positive-scale block completion, and each completion introduces at most one new numerical prediction value.
Since \(m_j=4^jm_0\), the total number of distinct candidate predictions is therefore at most
\[
1+|\barG|\sum_{j\ge1}\frac T{m_j}
=1+\frac{|\barG|T}{3m_0}
\le1+|\barG|T=A_{\max}.
\]
Reusing a numerical value or inserting it into additional candidate prediction sets does not create another first-discovery index.

\section{Analysis at a fixed resolution}
\subsection{Interval loss comparison}
\label{subsec:interval-loss-comparison-proof}

\begin{proof}[Proof of \Cref{lem:interval-loss-comparison}]
The local residual bound uses the interval comparison to control accumulated logarithmic-loss differences.
We derive it by using the stationarity equations to bound the total comparison weight, then bounding the product of updates for a fixed pair on an arbitrary available interval.

\textbf{Step 1. The stationarity equations control the total comparison weight.}
By the definition of smoothing,
\[
y'_{t,i}=(1-\beta)y_{t,i}+\frac{\beta}{n}
\ge\frac{\beta}{n}=: \epsilon.
\]
Every available prediction $a \in \mathcal R_t$ is either \(q_0\) or a weighted mean of smoothed outcomes, so each coordinate is at least \(\epsilon\).
Hence, by the fact that $y'_{t} \in \Delta_{n-1}$,
\[0\le\ell(a,y'_t) = \sum_i y'_{t,i} \log(1/a_i)\le\log(1/\epsilon) \sum_{i} y_{t,i}' = \log (1/\epsilon)=\log(n/\beta),\]
and
\[
|c_t^{g,a,b}|
\le g(x_t)\pi_{t,a}
\frac{|\ell(a,y'_t)-\ell(b,y'_t)|}{\log(n/\beta)}
\le g(x_t)\pi_{t,a}\le1.
\]
Since \(\eta<1/8\), all factors \(1+\eta c_t\) are positive and all aggregates are at least one.
The stationarity equations have a unique, strictly positive probability solution, as proved in \hyperlink{details-step-1}{Details of Step 1}.

Every newly discovered candidate other than \(q_0\) is tested for insertion at every coarser positive level.
If insertion is rejected at such a level, a candidate within the corresponding threshold remains there permanently because the candidate sets only grow.
Thus its coarsest membership is fixed at discovery, so once it leaves \(\mathcal R_t\) as \(K\) decreases, it cannot re-enter.
Comparison aggregates can therefore be discarded permanently once either prediction becomes unavailable.

For analysis, extend the comparison aggregates to all \(|\barG|A_{\max}^2\) group--prediction slots, setting \(c_t=0\) whenever either prediction is unavailable or the two predictions coincide.
The total potential sums over all these slots, but only distinct available pairs contribute to its increment.
Suppressing the time subscript and writing \(\ell_a=\ell(a,y'_t)\), we obtain
\[
\begin{aligned}
\sum_{g,a,b}B^{g,a,b}c^{g,a,b}
&=\sum_{\substack{a,b\in\mathcal R_t\\a\ne b}}\sum_{g\in\barG}
B^{g,a,b}\frac{g(x_t)\pi_a(\ell_a-\ell_b)}{\log(n/\beta)}\\
&=\frac1{\log(n/\beta)}\sum_{\substack{a,b\in\mathcal R_t\\a\ne b}}
\pi_a(\ell_a-\ell_b)
\underbrace{\sum_{g\in\barG}g(x_t)B^{g,a,b}}_{=\,W_{a,b}}\\
&=\frac1{\log(n/\beta)}\sum_{a,b\in\mathcal R_t}\pi_aW_{a,b}(\ell_a-\ell_b).
\end{aligned}
\]
The last equality inserts the zero diagonal terms \(W_{a,a}=0\).
By \eqref{eq:comparison-matrix},
\[
\begin{aligned}
\sum_{a,b\in\mathcal R_t}\pi_aW_{a,b}(\ell_a-\ell_b)
&=\sum_{a\in\mathcal R_t}\ell_a\pi_a\sum_{b\in\mathcal R_t}W_{a,b}
-\sum_{b\in\mathcal R_t}\ell_b\sum_{a\in\mathcal R_t}\pi_aW_{a,b}\\
&=\sum_{b\in\mathcal R_t}\ell_b
\left(\pi_b\sum_{\substack{c\in\mathcal R_t\\c\ne b}}W_{b,c}
-\sum_{\substack{a\in\mathcal R_t\\a\ne b}}\pi_aW_{a,b}\right)
=0.
\end{aligned}
\]
Thus
we have \(\sum_{g,a,b}B^{g,a,b}c^{g,a,b} = 0\).  Thus by (\ref{eq:comparison-potential-update}), we obtain
\[
\begin{aligned}
\sum_{g,a,b}B_{t+1}^{g,a,b}
&=|\barG|A_{\max}^2+\sum_{g,a,b}B_t^{g,a,b}
+\eta\underbrace{\sum_{g,a,b}B_t^{g,a,b}c_t^{g,a,b}}_{=0}\\
&=|\barG|A_{\max}^2+\sum_{g,a,b}B_t^{g,a,b}.
\end{aligned}
\]
Since \(B_1^{g,a,b}=1\) for every \(g,a,b\), we have
\begin{equation}\label{eq:btbound}
\sum_{g,a,b}B_t^{g,a,b}
=|\barG|A_{\max}^2+(t-1)|\barG|A_{\max}^2
=|\barG|A_{\max}^2t.
\end{equation}

\textbf{Step 2. Complete the interval loss-comparison bound.}
Fix \(g,a,b\) and a contiguous interval \(I=[u,v]\) on which both predictions are available.
Write \(X_t=\ell(a,y'_t)-\ell(b,y'_t)\).
If \(a=b\), then \(X_t=0\) and the claim is immediate.
Otherwise iterating (\ref{eq:comparison-potential-update}) gives
\[
B_{v+1}^{g,a,b}\ge B_u^{g,a,b}\prod_{t=u}^v(1+\eta c_t^{g,a,b}) \geq \prod_{t=u}^v(1+\eta c_t^{g,a,b}),
\]
where the last inequality is due to \(B_u^{g,a,b}\ge1\).  Thus,  by (\ref{eq:btbound}), we obtain
\[
\prod_{t=u}^{v}(1+\eta c_t^{g,a,b})
\le B_{v+1}^{g,a,b}
\le \sum_{g',a',b'}B_{v+1}^{g',a',b'}
=|\barG|A_{\max}^2(v+1)
\le|\barG|A_{\max}^2(T+1).
\]
Taking logarithms gives
\[
\sum_{t\in I}\log(1+\eta c_t^{g,a,b})
\le\log\!\bigl(|\barG|A_{\max}^2(T+1)\bigr).
\]
Since \(|\eta c_t^{g,a,b}|\le\eta<1/8<1/2\), applying \(\log(1+x)\ge x-x^2\) to each summand yields
\[
\eta\sum_{t\in I}c_t^{g,a,b}
-\eta^2\sum_{t\in I}(c_t^{g,a,b})^2
\le\log\!\bigl(|\barG|A_{\max}^2(T+1)\bigr).
\]
Substituting \(c_t^{g,a,b}=\alpha_tX_t/\log(n/\beta)\), rearranging, and multiplying by \(\log(n/\beta)/\eta\), we obtain
\[
\sum_{t\in I}\alpha_tX_t
\le\frac{\eta}{\log(n/\beta)}\sum_{t\in I}\alpha_t^2X_t^2
+\frac{\log(n/\beta)}{\eta}\log\!\bigl(|\barG|A_{\max}^2(T+1)\bigr).
\]
Finally, \(0\le\alpha_t=g(x_t)\pi_{t,a}\le1\) implies \(\alpha_t^2\le\alpha_t\).
Using \(\eta=\log(n/\beta)/(8\kappa)\) and \(\log(|\barG|A_{\max}^2(T+1))\le L\), we conclude that
\[
\sum_{t\in I}\alpha_tX_t
\le\frac1{8\kappa}\sum_{t\in I}\alpha_tX_t^2+8\kappa L.
\]
This proves \eqref{eq:second-order-comparison} on every path for all such intervals, including intervals selected after observation.

\hypertarget{details-step-1}{}
\textbf{Details of Step 1: Existence and uniqueness of the prediction distribution.}
If there is one available prediction, its probability is one.
Otherwise, positivity of the comparison weights and the constant group give \(W_{t,a,b}>0\) for every \(a\ne b\).
Set \(Z_t=\max_a\sum_{b\ne a}W_{t,a,b}\), and define
\[
M_t(a,b)=\frac{W_{t,a,b}}{Z_t}\quad(a\ne b),
\qquad
M_t(a,a)=1-\sum_{b\ne a}M_t(a,b).
\]
This matrix is row-stochastic and irreducible.
The Perron--Frobenius theorem therefore gives a unique, strictly positive stationary probability vector.
Its stationarity equations are equivalent to \eqref{eq:comparison-matrix}, proving the claim.

\end{proof}

\subsection{Local comparison}
We first establish the distance inequalities used in the available-comparator bound and the subsequent affine packing analysis.
\begin{lemma}[KL--Hellinger comparison]\label{lem:kl-distance-bounds}
Let \(p,q\in\Delta_{n-1}\) and \(0<\epsilon\le1\), with \(q_i\ge\epsilon\) for every \(i\).
Then
\[
H(p,q)^2\le D(p\Vert q)
\le\bigl(2+\log(1/\epsilon)\bigr)H(p,q)^2.
\]
\end{lemma}
\begin{proof}
For \(p_i>0\), the inequality \(\log x\le x-1\), applied to \(x=\sqrt{q_i/p_i}\), gives
\(p_i\log(p_i/q_i)\ge2p_i-2\sqrt{p_iq_i}\).
It also holds at \(p_i=0\) by continuity.
Summing gives \(D(p\Vert q)\ge2-2\sum_i\sqrt{p_iq_i}=H(p,q)^2\).

For the upper bound, \(\sum_i(p_i-q_i)=0\) gives
\[
D(p\Vert q)=\sum_i q_i f\bigl(\sqrt{p_i/q_i}\bigr),
\qquad f(x)=2x^2\log x-x^2+1.
\]
Here \(f(1)=0\), \(f'(x)=4x\log x\), and \(f(0)=1\) by continuity.
For \(0<x\le1\), use \(\log s\ge1-1/s\); for \(x\ge1\), use
\(s\log s\le(s-1)(1+\log x)\) on \(1\le s\le x\). Integrating yields
\[
\begin{aligned}
f(x)&=-4\int_x^1s\log s\,ds\le2(1-x)^2 &&(0\le x\le1),\\
f(x)&=4\int_1^x s\log s\,ds\le2(1+\log x)(x-1)^2 &&(x\ge1).
\end{aligned}
\]
For \(x=\sqrt{p_i/q_i}\ge1\), we have \(2\log x\le\log(1/\epsilon)\).
Thus in both cases \(f(x)\le(2+\log(1/\epsilon))(x-1)^2\).
Multiply by \(q_i\) and sum to obtain the upper bound.
\end{proof}

\begin{proof}[Proof of \Cref{lem:local-comparison}]
Fix a group \(g\) and a contiguous interval \(I\) on which predictions \(a,q\) are both available. Let \(\alpha_t=\pi_{t,a}g(x_t)\), \(N=\sum_{t\in I}\alpha_t>0\), and \(z=N^{-1}\sum_{t\in I}\alpha_ty'_t\).
We convert the interval loss-comparison bound of \Cref{lem:interval-loss-comparison} into a bound on \(D(z\Vert a)\) in terms of \(D(z\Vert q)\). To do so, we first control the squared loss differences by these two divergences, then substitute this estimate into the comparison bound and rearrange. The residual and Hellinger bounds follow from inequalities relating KL divergence to \(\ell_1\) and Hellinger distances.

\textbf{Step 1. Bound the squared loss difference between two predictions.}
For the two predictions \(a,q\), we define
\[
X_t=\ell(a,y'_t)-\ell(q,y'_t)
=\sum_i y'_{t,i}\log(q_i/a_i).
\]
Jensen inequality therefore gives
\[
\begin{aligned}
\sum_{t\in I}\alpha_tX_t^2
&\le\sum_{t\in I}\alpha_t\sum_i y'_{t,i}\log^2(q_i/a_i)\\
&=N\sum_i z_i\log^2(a_i/q_i),
\end{aligned}
\]
where the equality uses the definition \(z_i=\sum_{t\in I}\alpha_ty'_{t,i}/N\).
To bound the right-hand side, it therefore suffices to prove
\[
\sum_i z_i\log^2(a_i/q_i)\le2\kappa[D(z\Vert a)+D(z\Vert q)].
\]
To see which coordinatewise inequality is needed, first expand the two divergences.
Since \(a,q,z\) are probability vectors, \(\sum_i(a_i+q_i-2z_i)=0\), and hence
\[
D(z\Vert a)+D(z\Vert q)
=\sum_i\left[a_i+q_i-2z_i+z_i\log\frac{z_i^2}{a_iq_i}\right],
\]
where the logarithmic term is interpreted as zero when \(z_i=0\).
For \(z_i>0\), let \(u=a_i/z_i\) and \(v=q_i/z_i\).
The corresponding summand is \(z_i[u+v-2-\log(uv)]\), whereas the summand on the left of the desired bound is \(z_i\log^2(u/v)\).
It therefore suffices to prove
\[
\log^2(u/v)\le2\kappa[u+v-2-\log(uv)].
\]
The left-hand side depends only on the ratio \(u/v\).
Let \(t=\log(u/v)\). We use the following decomposition to obtain a lower bound depending only on \(t\):
\[
\begin{aligned}
u+v-2-\log(uv)
&=2\left(\frac{u+v}{2}-1-\log\frac{u+v}{2}\right)
+2\log\frac{u+v}{2\sqrt{uv}}\\
&\ge2\log\cosh(t/2).
\end{aligned}
\]
Here the inequality follows from \(x-1-\log x\ge0\) for \(x>0\), applied with \(x=(u+v)/2\), and \(\frac{u+v}{2\sqrt{uv}}=\cosh(t/2)\).
We next bound \(2\log\cosh(t/2)\) below in terms of \(t^2\).
For \(x\ge0\), the inequality \(e^{2x}\ge1+2x\) implies
\[
\tanh x=\frac{e^{2x}-1}{e^{2x}+1}\ge\frac{x}{1+x}.
\]
Integrating, we obtain
\[
\log\cosh x
=\int_0^x\tanh s\,ds
\ge\int_0^x\frac{s}{1+s}\,ds
\ge\int_0^x\frac{s}{1+x}\,ds
=\frac{x^2}{2(1+x)}.
\]
Taking \(x=|t|/2\) and using \(\cosh(t/2)=\cosh(|t|/2)\), we get
\[
2\log\cosh(t/2)\ge\frac{t^2}{4+2|t|}.
\]
Because \(t=\log(a_i/q_i)\) and \(a_i,q_i\in[\epsilon,1]\), we have \(|t|\le\log(1/\epsilon)=\log(n/\beta)\).
Thus \(4+2|t|\le4+2\log(n/\beta)=2\kappa\), giving
\[
\frac{a_i}{z_i} + \frac{q_i}{z_i} -2 - \log\left(\frac{a_iq_i}{z_i^2}\right) \geq \frac{t^2}{4+2|t|}\ge\frac{t^2}{2\kappa} = \frac{\log^2 (a_i/q_i)}{2\kappa}.
\]
Multiplying by \(z_i\), we obtain the coordinatewise bound
\[
z_i\log^2(a_i/q_i)
\le2\kappa\left[a_i+q_i-2z_i+z_i\log\frac{z_i^2}{a_iq_i}\right].
\]
When \(z_i=0\), the left-hand side is zero and the right-hand side is \(2\kappa(a_i+q_i)\ge0\), so the same bound holds.
Summing over all coordinates and using the divergence expansion above yields
\[
\sum_i z_i\log^2(a_i/q_i)\le2\kappa[D(z\Vert a)+D(z\Vert q)].
\]

Substituting into the initial Jensen bound gives
\[
\sum_{t\in I}\alpha_tX_t^2\le2\kappa N[D(z\Vert a)+D(z\Vert q)].
\]

\textbf{Step 2. Convert the interval comparison into a divergence bound.}
Both predictions are available on the whole interval, so \Cref{lem:interval-loss-comparison} applies.
Using the definition of \(z\), this comparison bound, and then the squared loss bound from Step 1, we obtain
\[
\begin{aligned}
N[D(z\Vert a)-D(z\Vert q)]
&=N\sum_i z_i\log(q_i/a_i)
=\sum_{t\in I}\alpha_tX_t\\
&\le\frac1{8\kappa}\sum_{t\in I}\alpha_tX_t^2+8\kappa L\\
&\le\tfrac14N[D(z\Vert a)+D(z\Vert q)]+8\kappa L.
\end{aligned}
\]
Since both divergences are nonnegative, replacing \(1/4\) by \(1/2\) for simpler constants preserves the upper bound:
\[
N[D(z\Vert a)-D(z\Vert q)]
\le\tfrac12N[D(z\Vert a)+D(z\Vert q)]+8\kappa L.
\]
Rearranging the last inequality gives
\[
D(z\Vert a)\le3D(z\Vert q)+\frac{16\kappa L}N.
\]

\textbf{Step 3. Complete the residual and Hellinger bounds.}
We apply \Cref{lem:kl-distance-bounds} to the divergence bound from Step 2, using \(2+\log(1/\epsilon)=\kappa\).
\begin{itemize}
\item \emph{Derive the Hellinger bound
\eqref{eq:iclr-local-hellinger}.}
Since every coordinate of \(a\) and \(q\) is at least \(\epsilon\), by Lemma~\ref{lem:kl-distance-bounds}, we have
\[
\begin{aligned}
H(z,a)^2
&\le D(z\Vert a)\\
&\le 3D(z\Vert q)+\frac{16\kappa L}{N}\\
&\le 3\kappa H(z,q)^2+\frac{16\kappa L}{N}.
\end{aligned}
\]
Taking square roots and using
\(\sqrt{x+y}\le\sqrt{x}+\sqrt{y}\) for \(x,y\ge0\) gives
\[
H(z,a)
\le\sqrt{3\kappa}H(z,q)+4\sqrt{\kappa L/N},
\]
which proves \eqref{eq:iclr-local-hellinger}.

\item \emph{Derive the residual bound
\eqref{eq:iclr-block-residual}.} Fix a scale \(j\) such that \(N\ge m_j\) and \(H(z,q)\le2 w_j\).
By Pinsker's inequality, the bound from Step 2, and \Cref{lem:kl-distance-bounds},
\[
\begin{aligned}
\|z-a\|_1^2
&\le 2D(z\Vert a)\\
&\le 6D(z\Vert q)+\frac{32\kappa L}{N}\\
&\le 6\kappa H(z,q)^2+\frac{32\kappa L}{N}.
\end{aligned}
\]
Taking square roots gives
\[
\|z-a\|_1
\le\sqrt{6\kappa}H(z,q)+\sqrt{32\kappa L/N}.
\]
If \(H(z,q)\le2w_j\) and
\(N\ge m_j=1024nL/w_j^2\), then
\[
\begin{aligned}
\|z-a\|_1
&\le2\sqrt{6\kappa}\,w_j
+\sqrt{\frac{32\kappa Lw_j^2}{1024nL}}\\
&=\left(2\sqrt6+\sqrt{\frac{32}{1024n}}\right)
\sqrt{\kappa}\,w_j\\
&\le6\sqrt{\kappa}\,w_j.
\end{aligned}
\]
This proves \eqref{eq:iclr-block-residual}.
For a completed scale-zero block, the same argument applies
with \(q=q_0\), since \(N\ge m_0\) and
\(H(z,q_0)\le\sqrt2\le2w_0\).
\end{itemize}
\end{proof}

\subsection{A common high-probability event}

Recall $\calF_t^-$ is the filtration before prediction is sampled from $\pi_t$ as defined in Section~\ref{def:uniform-model}. Let \(\calF_t\) contain the prediction, outcome, and all updates through round \(t\).
Then \(\calF_t^-\subseteq\calF_t\), and \(\calF_t\subseteq\calF_{t+1}^-\) for \(t<T\).

Let \(\mu'_t=\E[y'_t\mid\calF_t^-]=(1-\beta)\mu_t+\beta q_0\).
For a group \(g\in\barG\), candidate prediction \(a\) discovered during the run, and contiguous interval \(I\subseteq[T]\), let \(\alpha_t=g(x_t)\pi_{t,a}\) and \(N=\sum_{t\in I}\alpha_t>0\), and write
\[
z=\frac1N\sum_{t\in I}\alpha_t y'_t,
\qquad
\bar\mu'=\frac1N\sum_{t\in I}\alpha_t\mu'_t.
\]

\begin{lemma}[Outcome noise and prediction sampling]\label{lem:common-event}
There are events \(\mathcal E_{\rm conc}\) and \(\mathcal E_{\rm samp}\), each with failure probability at most \(\delta/2\), on which the following bounds hold simultaneously on the actual algorithmic run.
\begin{enumerate}[label=(\roman*),leftmargin=*]
\item On \(\mathcal E_{\rm conc}\), for every \(g,a,I\) with \(N>0\),
\[
H(z,\bar\mu')\le2\sqrt{nL/N}.
\]

In particular, a completed scale-\(j\) block \(E\) satisfies
\begin{equation}
H(z_E,\bar\mu'_E)\le w_j/16,
\label{eq:completed-block-concentration}
\end{equation}
where \(\bar\mu'_E\) is the average above with the weights and observation interval of \(E\).
\item On \(\mathcal E_{\rm samp}\), for every contiguous interval \(I\), every group \(g\in\barG\), and every set \(Q\) of \mbox{\(m\ge1\)} distinct candidate predictions discovered during the run,
\begin{equation}
\sum_{a\in Q}\left\|\sum_{t\in I}g(x_t)
(\mathbf1\{p_t=a\}-\pi_{t,a})(y'_t-a)\right\|_1
\le4\sqrt{nmLT}.
\label{eq:iclr-prediction-sampling}
\end{equation}
\end{enumerate}
Thus \(\mathcal E_{\rm global}:=\mathcal E_{\rm conc}\cap\mathcal E_{\rm samp}\) has probability at least \(1-\delta\).
\end{lemma}

We use the following two-sided consequence of the scalar Freedman inequality \citep{freedman1975tail,tropp2011freedman}.

\begin{lemma}[Freedman's inequality]\label{lem:scalar-freedman}
For a positive integer \(q\), let \((\mathcal H_t)_{t=0}^q\) be a filtration and let \(S_t=\sum_{v=1}^t\xi_v\), where \(\xi_v\) is \(\mathcal H_v\)-measurable, \(\E[\xi_v\mid\mathcal H_{v-1}]=0\), and \(|\xi_v|\le1\) almost surely.
Define \(W_t=\sum_{v=1}^t\E[\xi_v^2\mid\mathcal H_{v-1}]\).
For every \(b,x>0\),
\[
\Pr\!\left(\exists t\in[q]:\ |S_t|\ge\sqrt{2bx}+\frac{2x}{3},\ W_t\le b\right)
\le2e^{-x}.
\]
\end{lemma}

\begin{proof}
The scalar inequality in \citet[Theorem~1.1]{tropp2011freedman}, applied to \(S_t\) and \(-S_t\), gives
\[
\Pr\!\left(\exists t\in[q]:\ |S_t|\ge a,\ W_t\le b\right)
\le2\exp\!\left(-\frac{a^2}{2(b+a/3)}\right).
\]
For \(a=\sqrt{2bx}+2x/3\), we have \(a^2-2x(b+a/3)=(2x/3)\sqrt{2bx}\ge0\), proving the claim.
\end{proof}

We also need a one-sided, conditional exponential-moment version of the Hoeffding--Azuma maximal inequality (see \citet[Theorem~26.9]{tropp2024probability} for the bounded-increment formulation).
We state and prove the version needed here to make its assumptions and constants explicit.

\begin{lemma}[Maximal Azuma--Hoeffding inequality]\label{lem:maximal-azuma-hoeffding}
For a positive integer \(q\), let \((S_t)_{t=0}^q\) be a martingale with respect to a filtration \((\mathcal H_t)_{t=0}^q\), with \(S_0=0\).
Suppose that, for a deterministic \(c>0\), its increments \(\xi_t=S_t-S_{t-1}\) satisfy
\[
\E[e^{\lambda\xi_t}\mid\mathcal H_{t-1}]
\le e^{\lambda^2c^2/8}
\qquad(\lambda\in\mathbb R,\ t\in[q]).
\]
Then, for every \(x>0\),
\[
\Pr\!\left(\max_{1\le t\le q}S_t\ge c\sqrt{\frac{qx}{2}}\right)
\le e^{-x}.
\]
The exponential-moment assumption holds whenever each increment has conditional mean zero and, given \(\mathcal H_{t-1}\), lies in an interval of length at most \(c\).
\end{lemma}

\begin{proof}
Condition on \(\mathcal H_{t-1}\) and let \(f(\lambda)=\log\E[e^{\lambda\xi_t}\mid\mathcal H_{t-1}]\).
If \(\xi_t\) lies in an interval of length \(c\), then \(f''(\lambda)\) is its variance under exponential reweighting, at most \(c^2/4\): squared distance from the interval's midpoint is at most \(c^2/4\), and variance is no larger.
Since \(f(0)=f'(0)=0\), integration gives \(f(\lambda)\le\lambda^2c^2/8\) for every real \(\lambda\), proving the stated sufficient condition.
Under the exponential-moment assumption, for each \(\lambda>0\),
\[
Z_t=\exp\!\left\{\lambda S_t-\frac{t\lambda^2c^2}{8}\right\},
\qquad t=0,\ldots,q,
\]
is a nonnegative supermartingale with \(Z_0=1\).
Fix \(a>0\) and define the bounded stopping time
\[
\tau=\min\bigl(\{t\in[q]:S_t\ge a\}\cup\{q\}\bigr).
\]
On \(\{\max_{1\le t\le q}S_t\ge a\}\), we have \(S_\tau\ge a\) and \(\tau\le q\).
Consequently,
\[
1\ge\E[Z_\tau]
\ge\exp\!\left\{\lambda a-\frac{q\lambda^2c^2}{8}\right\}
\Pr\!\left(\max_{1\le t\le q}S_t\ge a\right).
\]
Taking \(\lambda=4a/(qc^2)\) gives a probability bound of \(\exp\{-2a^2/(qc^2)\}\).
Substituting \(a=c\sqrt{qx/2}\) proves the claim.
\end{proof}

\begin{proof}[Proof of \Cref{lem:common-event}]
Appendix~\ref{sec:discovery-hierarchy} bounds the number of distinct candidate predictions by \(A_{\max}\) on every run.
We prove the two concentration bounds separately.
For the shifted martingales in Steps~1 and~2, define \(\calF_{T+1}^-:=\calF_T\).

The probability proofs use permanent first-discovery indices in \([A_{\max}]\), one for each distinct candidate prediction.
On rounds when a candidate prediction is available, its value and mixture weight are measurable before prediction sampling.

\textbf{Step 1. Outcome-noise event \(\mathcal E_{\rm conc}\).}

\begin{enumerate}[label=(\alph*),leftmargin=*]
\item \emph{Define and normalize the weighted sums.}
For each group \(g\in\barG\), index \(k\in[A_{\max}]\), and round \(u\in[T]\), consider the following weighted sums.
Writing \(a_k\) for the value at index \(k\), use predictable weights \(\alpha_v=\one\{v\ge u\}\pi_{v,a_k}g(x_v)\), which is zero before that value is discovered or while it is unavailable.
For each endpoint \(t\in\{0,\ldots,T\}\), define
\[
N_t=\sum_{v=1}^t\alpha_v,\qquad
V_{t,i}=\sum_{v=1}^t\alpha_v y'_{v,i},\qquad
U_{t,i}=\sum_{v=1}^t\alpha_v\E[y'_{v,i}\mid\calF_v^-]\quad(i\in[n]).
\]
Write \(V_t=(V_{t,i})_{i=1}^n\) and \(U_t=(U_{t,i})_{i=1}^n\).
Since \(\alpha_v\ge0\) and \(y'_v\in\DeltaD\), both vectors have nonnegative coordinates and
\[
\begin{aligned}
\sum_{i=1}^n V_{t,i}
&=\sum_{v=1}^t\alpha_v\sum_{i=1}^n y'_{v,i}
=\sum_{v=1}^t\alpha_v=N_t,\\
\sum_{i=1}^n U_{t,i}
&=\sum_{v=1}^t\alpha_v\E\!\left[\sum_{i=1}^n y'_{v,i}\,\middle|\,\calF_v^-\right]
=\sum_{v=1}^t\alpha_v=N_t.
\end{aligned}
\]
Thus \(V_t/N_t,U_t/N_t\in\DeltaD\) whenever \(N_t>0\).

\item \emph{Apply Freedman to the weighted coordinate sums.}
For each coordinate \(i\), the increments
\(\xi_v:=\alpha_v(y'_{v,i}-\E[y'_{v,i}\mid\calF_v^-])\)
are conditionally centered given \(\calF_v^-\), i.e., \(\mathbb E[\xi_v\mid\mathcal F_v^-]=0\) and have absolute value bounded by \(1\):
\[
|\xi_v|
=\alpha_v\left|y'_{v,i}-\mathbb E[y'_{v,i}\mid\mathcal F_v^-]\right|
\le\alpha_v\le1.
\]
To use the notation of \Cref{lem:scalar-freedman}, define
\[
S_t:= V_{t,i}-U_{t,i}=\sum_{v=1}^t\xi_v,
\qquad
t=0,\ldots,T.
\]
Then \(S_0=0\), and \(S_t\) is measurable with respect to \(\mathcal H_t:=\calF_{t+1}^-\).
For \(t\in[T]\),
\[
\E[S_t-S_{t-1}\mid\mathcal H_{t-1}]
=\E[\xi_t\mid\calF_t^-]=0.
\]
Thus \((S_t)_{t=0}^T\) is a martingale with respect to \((\mathcal H_t)_{t=0}^T\).
Its cumulative conditional variance after round \(t\) satisfies
\[
\begin{aligned}
W_{t,i}
&:=\sum_{v=1}^t\E[\xi_v^2\mid\mathcal H_{v-1}]
=\sum_{v=1}^t
\E\!\left[\alpha_v^2
\bigl(y'_{v,i}-\E[y'_{v,i}\mid\calF_v^-]\bigr)^2
\,\middle|\,\calF_v^-\right]\\
&\le\sum_{v=1}^t\alpha_v^2\E[(y'_{v,i})^2\mid\calF_v^-]
\le\sum_{v=1}^t\alpha_v\E[y'_{v,i}\mid\calF_v^-]
=U_{t,i}.
\end{aligned}
\]
Thus, for every \(b,x>0\), \Cref{lem:scalar-freedman} gives
\[
\Pr\!\left(\exists t\in[T]:
|V_{t,i}-U_{t,i}|\ge\sqrt{2bx}+\frac{2x}{3},\
U_{t,i}\le b\right)\le2e^{-x}.
\]

\item \emph{Obtain coordinate bounds uniformly over all intervals.}
We apply geometric peeling to the cumulative conditional mean \(U_{t,i}\in[0,T]\), using the intervals
\[
[0,1],\qquad
\bigl((4/3)^{h-1},(4/3)^h\bigr],
\quad h=1,\ldots,J-1,
\]
where \(J=1+\lceil\log_{4/3}T\rceil\).
For each interval, we apply Freedman's inequality with its deterministic upper endpoint \(b_h=(4/3)^h\), and then take a union bound over the intervals.
Then, let
\[
L_0=\log\!\left(\frac{4n|\barG|A_{\max}TJ}{\delta}\right).
\]
Note that \(L_0>1\).
Apply the preceding Freedman bound with \(b=b_h\) and \(x=L_0\), and take a union bound over all \(n\) coordinates, \(|\barG|\) groups, \(A_{\max}\) prediction indices, \(T\) starts, and \(J\) variance bounds.
The total failure probability is at most
\[
2n|\barG|A_{\max}TJ e^{-L_0}=\delta/2.
\]
Let \(\mathcal E_{\rm conc}\) be the simultaneous success event.
On this event, for every group, prediction index, start, endpoint \(t\), coordinate \(i\), and \(h\), if \(U_{t,i}\le b_h\), then
\[
|V_{t,i}-U_{t,i}|
\le\sqrt{2b_hL_0}+\frac{2L_0}{3}.
\]

\item \emph{Convert coordinate bounds to a Hellinger bound.}
Fix an endpoint \(t\) with \(N_t>0\) and suppress its subscript in \(N_t,V_t,U_t\) and their coordinates.
The empirical mean and weighted predictable mean are \(z=V/N\) and \(\bar\mu'=U/N\).
Our next goal is to prove
\[
H(z,\bar\mu')^2
=\frac1N\sum_{i=1}^n(\sqrt{V_i}-\sqrt{U_i})^2
\le\frac{4nL_0}{N}.
\]
It suffices to prove \(\lvert\sqrt{V_i}-\sqrt{U_i}\rvert\le2\sqrt{L_0}\) for every \(i\).

If \(U_i\le1\), use \(b_0=1\).
Since \(L_0>1\),
\[
0\le V_i\le1+\sqrt{2L_0}+\frac{2L_0}{3}<4L_0,
\]
so both square roots lie in \([0,2\sqrt{L_0}]\).

If \(U_i>1\), choose the smallest \(h\) with \(b_h\ge U_i\).
Such an \(h\) exists because \(U_i\le T\), and \(U_i\le b_h<(4/3)U_i\).
Hence
\[
|V_i-U_i|\le\sqrt{\frac83U_iL_0}+\frac{2L_0}{3}.
\]
For the upper deviation, this implies
\[
V_i\le U_i+\sqrt{\frac83U_iL_0}+\frac{2L_0}{3}
\le(\sqrt{U_i}+2\sqrt{L_0})^2,
\]
and therefore \(\sqrt{V_i}-\sqrt{U_i}\le2\sqrt{L_0}\).
For the lower deviation, consider \(V_i<U_i\).
If \(U_i\le4L_0\), then \(\sqrt{U_i}-\sqrt{V_i}\le\sqrt{U_i}\le2\sqrt{L_0}\).
Otherwise,
\[
\begin{aligned}
\sqrt{U_i}-\sqrt{V_i}
&=\frac{U_i-V_i}{\sqrt{U_i}+\sqrt{V_i}}\\
&\le \frac{\sqrt{\frac{8}{3}U_iL_0}+\frac{2L_0}{3}}{\sqrt{U_i}} \\
&\le\sqrt{\frac83L_0}+\frac{2L_0}{3\sqrt{U_i}} \\
&\le\left(\sqrt{\frac83}+\frac13\right)\sqrt{L_0} \\
&<2\sqrt{L_0}.
\end{aligned}
\]
Combining the two directions and summing over coordinates gives
\[
H(z,\bar\mu')^2
=\frac1N\sum_{i=1}^n(\sqrt{V_i}-\sqrt{U_i})^2
\le\frac{4nL_0}{N}.
\]
To apply this bound at the block weight threshold \(m_\ell=1024nL/w_\ell^2\), we check that \(L_0\le L\).
Since \((4/3)^3>2\), we have \(J\le3(1+\lceil\log_2T\rceil)\).
The definition of \(L\) in Appendix~\ref{app:universal-constants} therefore gives
\[
L_0
\le\log\!\left(\frac{12n|\barG|A_{\max}T(1+\lceil\log_2T\rceil)}{\delta}\right)
\le L.
\]
Consequently, for \(N\ge m_\ell\),
\[
H(z,\bar\mu')
\le2\sqrt{\frac{nL_0}{N}}
\le\frac{w_\ell}{16}\sqrt{\frac{L_0}{L}}
\le\frac{w_\ell}{16}.
\]
All starts and endpoints were covered before choosing a block, so the bound includes blocks ending at a depth change or the horizon.
\end{enumerate}

\textbf{Step 2. Prediction-sampling event \(\mathcal E_{\rm samp}\).}
\begin{enumerate}[label=(\alph*),leftmargin=*]
\item \emph{Express sampling residuals as centered scalar increments.}
Fix a deterministic start \(u\in[T]\), group \(g\in\barG\), subset \(Q\subseteq[A_{\max}]\) of size \(m\), and signs \(\sigma_i\in\{-1,1\}^n\) for \(i\in Q\).
For this proof, let \(a_i\) denote the candidate prediction with index \(i\).
Set its summand to zero until it is discovered and whenever it is unavailable.
For each \(t\in[T]\), the scalar increment
\[
\xi_t=g(x_t)\sum_{i\in Q}
\left(\one\{p_t=a_i\}-\pi_{t,a_i}\right)
\sigma_i^{\mathsf T}(y'_t-a_i)
\]
is conditionally centered given \(\calF_t^-\) and \(y'_t\).

\item \emph{Construct the martingale and its exponential-moment bound.}
Given \(\calF_t^-\) and \(y'_t\), define
\[
Y_t=g(x_t)\sum_{i\in Q}\one\{p_t=a_i\}\sigma_i^{\mathsf T}(y'_t-a_i).
\]
At most one summand is nonzero, and \(\|y'_t-a_i\|_1\le2\), so \(Y_t\in[-2,2]\).
Private-draw independence gives
\[
\xi_t=Y_t-\E[Y_t\mid\calF_t^-,y'_t].
\]
By the conditional Hoeffding bound established in the proof of \Cref{lem:maximal-azuma-hoeffding}, the value-range length \(c=4\) gives
\[
\E[e^{\lambda \xi_t}\mid\calF_t^-,y'_t]\le e^{2\lambda^2}.
\]
Taking conditional expectations again yields
\[
\E[\xi_t\mid\calF_t^-]=0,
\qquad
\E[e^{\lambda \xi_t}\mid\calF_t^-]\le e^{2\lambda^2}.
\]
As in Step~1, use the pre-sampling filtration with a one-round shift and define
\[
S_t=\sum_{v=1}^{t}\one\{v\ge u\}\xi_v,
\qquad t=0,\ldots,T.
\]
Define \(\mathcal H_t:=\calF_{t+1}^-\) for \(t=0,\ldots,T\).
Then \(S_0=0\), \(S_t\) is \(\mathcal H_t\)-measurable, and, for \(t\in[T]\),
\[
\E[S_t-S_{t-1}\mid\mathcal H_{t-1}]=0,
\qquad
\E[e^{\lambda(S_t-S_{t-1})}\mid\mathcal H_{t-1}]\le e^{2\lambda^2}.
\]
Thus \((S_t)_{t=0}^T\) is a martingale with respect to \((\mathcal H_t)_{t=0}^T\), and \(S_t=\sum_{v=u}^t\xi_v\) for \(t\ge u\).

\item \emph{Obtain a simultaneous bound over intervals, groups, subsets, and signs.}
Set \(A_m=nm\log2+m\log A_{\max}+\log(2|\barG|TA_{\max}/\delta)\).
Apply \Cref{lem:maximal-azuma-hoeffding} to this martingale with \(q=T\), \(c=4\), and \(x=A_m\).
This gives
\[
\Pr\!\left(\max_{u\le t\le T}\sum_{v=u}^t\xi_v\ge\sqrt{8TA_m}\right)
\le e^{-A_m}.
\]
For fixed \(m\), there are at most \(\binom{A_{\max}}m\le A_{\max}^m\) index subsets, \(2^{nm}\) sign choices, \(T\) starts, and \(|\barG|\) groups.
The union bound over these choices gives a total failure probability at most
\[
|\barG|T A_{\max}^m2^{nm}e^{-A_m}
=\frac{\delta}{2A_{\max}}.
\]
Summing over \(m=1,\ldots,A_{\max}\) gives total failure probability at most \(\delta/2\).

\item \emph{Recover the \(\ell_1\) sampling-residual bound.}
On the simultaneous success event from part~(c), maximize over the sign vector for each prediction and use \(A_m\le2nmL\).
For every interval \([u,t]\subseteq[T]\), group \(g\in\barG\), and set \(Q\) of \(m\ge1\) distinct candidate predictions discovered during the run, this gives
\[
\begin{aligned}
\sum_{a\in Q}\left\|\sum_{v=u}^t
g(x_v)(\one\{p_v=a\}-\pi_{v,a})(y'_v-a)\right\|_1
&\le\sqrt{8TA_m}\\
&\le4\sqrt{nmLT}.
\end{aligned}
\]
Denote this event by \(\mathcal E_{\rm samp}\), which has probability at least \(1-\delta/2\).
\end{enumerate}
By a union bound, \(\mathcal E_{\rm global}=\mathcal E_{\rm conc}\cap\mathcal E_{\rm samp}\) has probability at least \(1-\delta\).

\end{proof}

\subsection{Affine-dimensional Hellinger geometry}
\label{subsec:affine-hellinger-geometry}

For the remaining packing and residual analysis, fix a realization in \(\mathcal E_{\rm global}\), whose definition is independent of the choice of \(s,S\).
Taking coordinatewise square roots turns Hellinger distance into Euclidean distance, but the image of a set of affine dimension \(s\) need not lie in an \(s\)-dimensional affine subspace.
We instead parameterize the set in its own affine span and combine these coordinates with the gradient of negative entropy.
The resulting map keeps points in a bounded union of Euclidean balls and increases their separation, so a Euclidean packing bound applies in dimension \(s\).

\begin{definition}[Hellinger packing numbers]
\label{def:l1-packing}
Let \(S\subseteq\Delta_{n-1}\), \(v>0\), and \(R>0\). Define
\[
\GlobPack_H(S;v) := \sup\left\{ |Q|:\ Q\subseteq S,\ H(q,q')>v/2
\text{ for all distinct }q,q'\in Q \right\}
\qquad\text{(global packing)},
\]
and
\[
\LocPack_H(S;v,R) := \sup_{p\in\Delta_{n-1}} \GlobPack_H\!\left(S\cap\{q\in\Delta_{n-1}:H(p,q)\le Rv\};v\right)
\qquad\text{(local packing)}.
\]
Both numbers use separation \(v/2\). The global number counts
separated points across \(S\). The local number is the largest such count within any ball of radius \(Rv\), where \(R\) specifies the ball radius relative to the scale \(v\).
\end{definition}

\begin{lemma}[Ellipsoid covers of a convex body and its polar]
\label{lem:ellipsoid-covers}
There is a universal constant \(C_0\ge2\) such that,
for every integer \(s\ge1\) and every origin-symmetric
convex body \(B\subseteq\mathbb R^s\), there is an
origin-centered ellipsoid \(E\) such that \(B\) can be
covered by at most \(C_0^s\) translates of \(E\), and
\(B^\circ\) can be covered by at most \(C_0^s\) translates
of \(E^\circ\).
Here, \(B^\circ\) and \(E^\circ\) denote the polar bodies
of \(B\) and \(E\), respectively, where
\(K^\circ:=\{y\in\mathbb R^s:
\langle x,y\rangle\le1\text{ for all }x\in K\}\).
\end{lemma}
\begin{proof}
We use Milman's ellipsoid theorem in the formulation of \citet[Theorem~1.1]{bobkov2011milman}.
There is an origin-centered ellipsoid \(E\) such that
\[
\frac{|B+E|}{|B\cap E|}\frac{|B^\circ+E^\circ|}{|B^\circ\cap E^\circ|}\le C_M^s
\]
for a universal \(C_M\), where \(|\cdot|\) denotes \(s\)-dimensional volume and \(+\) denotes Minkowski addition.
Let \(x_1,\ldots,x_m\in B\) be the centers of a maximal family of translates \(x_i+E/2\) with pairwise disjoint interiors.
Each translate lies in \(B+E\), and volume becomes \(2^{-s} |E|\).
Since their boundaries have zero volume,
\[
m\,2^{-s}|E|
=
\left|\bigcup_{i=1}^m(x_i+E/2)\right|
\le |B+E|.
\]
Moreover, each factor in Milman's inequality is at least one, and \(B\cap E\subseteq E\), so
\[
\frac{|B+E|}{|E|}
\le
\frac{|B+E|}{|B\cap E|}
\le C_M^s.
\]
Consequently,
\[
m\le 2^s\frac{|B+E|}{|E|}
\le (2C_M)^s.
\]
Maximality makes the corresponding translates of \(E\) cover \(B\).
The same argument applies to the polar pair.
Set \(C_0=\max(2,2C_M)\).
\end{proof}

\begin{lemma}[Dimension-uniform Hellinger packing]\label{lem:h-affine-packing}
There is a universal constant \(C_0\ge2\) such that, for every positive compact convex probability set \(S\) of affine dimension \(s\) and every \(D_0>0\) satisfying \(D(p\Vert q)\le D_0\) for all \(p,q\in S\), the following bound holds for every \(v>0\):
\[
\GlobPack_H(S;v)\le[C_0^2(1+2\sqrt{2D_0}/v)]^s.
\]
In particular, for coordinate floor \(\epsilon\),
\[
\GlobPack_H(S;v)\le[C_0^2(1+2\sqrt{2\kappa}/v)]^s,
\qquad\LocPack_H(S;v,R)\le[C_0^2(1+4\sqrt{2\kappa}R)]^s.
\]
\end{lemma}
\begin{proof}
We reduce Hellinger packing in \(S\) to Euclidean packing in \(\mathbb R^s\).
Step 1 maps points at Hellinger distance greater than \(v/2\) to points at Euclidean distance greater than \(2v\).
Step 2 uses \Cref{lem:ellipsoid-covers} to cover the image by at most \(C_0^{2s}\) balls of radius \(2\sqrt{2D_0}\), then counts the separated image points by comparing volumes.
Step 3 specializes the divergence diameter to obtain the global and local estimates.
If \(s=0\), then \(S\) is a singleton and all the asserted packing bounds hold.
Henceforth assume \(s\ge1\).

\textbf{Step 1. Construct a map that expands Hellinger separation.}
The construction uses the fact that symmetrized KL divergence can be written as an inner product between a difference of probabilities and a difference of their logarithms.
We first express this identity in \(s\)-dimensional coordinates, then turn the inner product into a lower bound on a Euclidean distance squared.

Choose \(p_0\in S\).
Since \(S\) has affine dimension \(s\), the linear space
\(\operatorname{span}\{p-p_0:p\in S\}\) has dimension \(s\).
Choose a basis \(a_1,\ldots,a_s\) of this space and define
\(A:=[a_1\ \cdots\ a_s]\in\mathbb R^{n\times s}\).
Define
\[
p(x):=p_0+Ax \quad (x\in\mathbb R^s),
\qquad
X:=\{x\in\mathbb R^s:p(x)\in S\}.
\]
Then every point in \(S\) has a unique representation
\(p(x)\) with \(x\in X\).

The set \(X\) is compact, convex, full-dimensional in \(\mathbb R^s\), and contains zero.
Define \(g(x)=A^{\mathsf T}\log p(x)\), where the logarithm is taken coordinatewise.
This is the gradient of negative entropy \(\sum_i p_i(x)\log p_i(x)\), since each column of \(A\) sums to zero.
Since \(A(x-y)=p(x)-p(y)\), we have
\[
\begin{aligned}
\langle x-y,g(x)-g(y)\rangle
&=\sum_i\bigl(p_i(x)-p_i(y)\bigr)\log\frac{p_i(x)}{p_i(y)}\\
&=D(p(x)\Vert p(y))+D(p(y)\Vert p(x)).
\end{aligned}
\]
To turn this inner product into a distance bound, we will use \(\|r+s\|^2\ge4\langle r,s\rangle\), which follows from \(\|r-s\|^2\ge0\).
Thus we want to add transformed versions of \(x\) and \(g(x)\) whose differences preserve the inner product above.

Define \(B=X-X:=\{x-y:x\in X,\ y\in X\}\), an origin-symmetric convex body.
Apply \Cref{lem:ellipsoid-covers} to \(B\), and let \(E\) be the resulting ellipsoid, with covering constant \(C_0\).
Choose an invertible linear map \(L\) taking \(E\) to the Euclidean unit ball.
Then \(L^{-\mathsf T}\) takes \(E^\circ\) to the same unit ball, and the two transformations preserve inner products: \(\langle Lr,L^{-\mathsf T}s\rangle=\langle r,s\rangle\).
This choice will also let us use the ellipsoid covers to bound the image in Step 2.
Define
\[
F(x):=\sqrt{2D_0}\,Lx+(2D_0)^{-1/2}L^{-\mathsf T}(g(x)-g(0)),
\]
where the two scaling factors multiply to one, so they preserve the inner product.
As shown in Step 2, they also give the two summands the same covering radius.
Applying the preceding norm inequality to the two summands of \(F(x)-F(y)\) gives
\[
\begin{aligned}
\|F(x)-F(y)\|^2
&\ge4\langle x-y,g(x)-g(y)\rangle\\
&=4[D(p(x)\Vert p(y))+D(p(y)\Vert p(x))]\\
&\ge16H(p(x),p(y))^2.
\end{aligned}
\]
The last inequality follows by summing \((a-b)\log(a/b)\ge4(\sqrt a-\sqrt b)^2\) over coordinates.
For completeness, when \(r\ge1\), the function \(\log r-2(r-1)/(r+1)\) vanishes at \(1\) and has derivative \((r-1)^2/[r(r+1)^2]\ge0\).
Taking \(r=\sqrt{a/b}\) proves the scalar inequality for \(a\ge b\), and symmetry gives the other case.

Let \(Q\subseteq S\) be any finite set with pairwise Hellinger distances greater than \(v/2\).
Its corresponding image points under \(F\) therefore have pairwise Euclidean distances greater than \(2v\).
It remains to bound how many such points the image can contain.

\textbf{Step 2. Cover the image and count the separated points.}
We will cover each of the two summands of \(F\) by at most \(C_0^s\) balls of radius \(\sqrt{2D_0}\).
Adding the centers of these two covers will then cover \(F(X)\) by at most \(C_0^{2s}\) balls of radius \(2\sqrt{2D_0}\).

Since \(0\in X\), we have \(X\subseteq B\), so the first summand is controlled by the cover of \(B\).
For the second summand, we first show that \(g(x)-g(0)\in2D_0B^\circ\).
For \(u,v,x\in X\), expanding the divergences gives
\[
\begin{aligned}
\langle g(x)-g(0),u-v\rangle
={}&D(p(u)\Vert p_0)+D(p(v)\Vert p(x))\\
&-D(p(u)\Vert p(x))-D(p(v)\Vert p_0).
\end{aligned}
\]
Each divergence lies in \([0,D_0]\), so the absolute value of this expression is at most \(2D_0\).
Since every element of \(B=X-X\) has the form \(u-v\), the definition of the polar body \(B^\circ
:=\{y\in\mathbb R^s:\langle y,b\rangle\le1
\text{ for every }b\in B\}\) implies
\[
g(x)-g(0)\in 2D_0B^\circ
\qquad(x\in X).
\]

The cover of \(B\) by at most \(C_0^s\) translates of \(E\), after applying \(\sqrt{2D_0}L\), covers the first summand by at most \(C_0^s\) Euclidean balls of radius \(\sqrt{2D_0}\).
Similarly, cover \(B^\circ\) by at most \(C_0^s\) translates of \(E^\circ\), scale by \(2D_0\), and apply \((2D_0)^{-1/2}L^{-\mathsf T}\).
The inclusion just proved shows that these balls cover the second summand, and their radius is \(2D_0(2D_0)^{-1/2}=\sqrt{2D_0}\).
Let \(c_1,\ldots,c_M\) and \(d_1,\ldots,d_N\) be the centers of the covers for the first and second summands, respectively, where \(M,N\le C_0^s\).
For every \(x\in X\), there exist \(i\) and \(j\) such that
\[
\left\|\sqrt{2D_0}Lx-c_i\right\|
\le\sqrt{2D_0},
\qquad
\left\|(2D_0)^{-1/2}L^{-\mathsf T}(g(x)-g(0))-d_j\right\|
\le\sqrt{2D_0}.
\]
The triangle inequality therefore gives
\[
\begin{aligned}
\|F(x)-(c_i+d_j)\|
&\le
\left\|\sqrt{2D_0}Lx-c_i\right\| +
\left\|(2D_0)^{-1/2}L^{-\mathsf T}(g(x)-g(0))-d_j\right\|\\
&\le 2\sqrt{2D_0}.
\end{aligned}
\]
Thus the balls centered at all pairwise sums \(c_i+d_j\), each of radius \(2\sqrt{2D_0}\), cover \(F(X)\).
There are at most \(MN\le C_0^{2s}\) such balls, which proves the desired covering bound.

We now count the image points corresponding to \(Q\) within each ball of this cover.
Fix one such covering ball, whose radius is \(2\sqrt{2D_0}\), and suppose it contains \(m\) image points.
The Euclidean balls of radius \(v\) centered at these points are pairwise disjoint.
By the triangle inequality, all these smaller balls lie in the concentric ball of radius \(2\sqrt{2D_0}+v\).
Writing \(\omega_s\) for the volume of the Euclidean unit ball in \(\mathbb R^s\), comparison of volumes gives
\[
m\,\omega_s v^s
\le \omega_s\bigl(2\sqrt{2D_0}+v\bigr)^s,
\qquad
m\le\left(1+\frac{2\sqrt{2D_0}}{v}\right)^s.
\]
Since at most \(C_0^{2s}\) covering balls contain all the image points, we obtain
\[
|Q|
\le C_0^{2s}\left(1+\frac{2\sqrt{2D_0}}{v}\right)^s.
\]
Taking the supremum over such \(Q\) proves the first packing bound.

\textbf{Step 3. Specialize the divergence diameter.}
\emph{Global bound.}
When every coordinate is at least \(\epsilon\), we have \(D(p\Vert q)\le\log(1/\epsilon)\le\kappa\) for all \(p,q\in S\).
Taking \(D_0=\kappa\) in the bound just proved gives
\[
\GlobPack_H(S;v)\le[C_0^2(1+2\sqrt{2\kappa}/v)]^s.
\]

\emph{Local bound.}
Fix a center \(p\in\Delta_{n-1}\) and consider \(K=S\cap\{q:H(p,q)\le Rv\}\).
This set is compact and convex, because \(H(p,q)^2\) is continuous and convex in \(q\).
For any \(q,q'\in K\), the triangle inequality and \Cref{lem:kl-distance-bounds} give
\[
H(q,q')\le2Rv,
\qquad
D(q\Vert q')\le\kappa H(q,q')^2\le4\kappa R^2v^2.
\]
If \(K\) is empty or a singleton, its packing count is at most one.
Otherwise \(R>0\), and we apply the bound from Step 2 to \(K\) in its own affine span, with \(D_0=4\kappa R^2v^2\) and the same separation \(v/2\).
Since \(\operatorname{affdim}(K)\le s\), this gives
\[
\GlobPack_H(K;v)
\le[C_0^2(1+4\sqrt{2\kappa}R)]^{\operatorname{affdim}(K)}
\le[C_0^2(1+4\sqrt{2\kappa}R)]^s.
\]
Taking the supremum over centers \(p\) proves the local bound.
\end{proof}

\subsection{Partitioning observations and counting pieces}\label{subsec:partition-piece-counts}

\paragraph{The partition rule.}
The local comparison requires an associated comparator that was available throughout the block.
For a completed scale-zero threshold block, set \(\mathcal P(E)=\{E\}\).
For a completed positive-level block with children \(E_1,\ldots,E_m\) in chronological order, set
\[
\mathcal P(E)=
\begin{cases}
\{E\},&E\text{ is not a boundary block},\\
\bigcup_{i=1}^m\mathcal P(E_i),&E\text{ is a boundary block}.
\end{cases}
\]
At a depth-interval endpoint, apply this map to the completed scale-\((K-1)\) roots and apply the child union to the unfinished positive-scale roots of the remaining forest.
Use only maximal roots, and add the unique unfinished positive-weight scale-zero successor of each \((g,a)\), if present.
Empty blocks contribute nothing.
This is the parent-or-children selection rule from Section~\ref{def:boundary-analysis}.
It is defined from the chronological updates for analysis and is not maintained by the algorithm.

For \(0\le j<K-1\) and each target group \(g\), count its completed scale-\(j\) residual pieces over all source predictions.
Let \(Q_j(I)\) be the largest of these counts.

\paragraph{Counting each observation once.}
The parent-or-children recursion selects a disjoint partition.
Each weighted observation belongs to one depth interval and one source; each selected completed piece below the finest level has a unique parent, which is boundary or unfinished.

\begin{lemma}[Partition and piece counts]\label{lem:residual-charging}
On \(\mathcal E_{\rm global}\), let \(I\) be any nonempty prefix of a depth-\(K\) interval.
Let \(M_I\) and \(B_I\) be the counts defined in \Cref{def:discovery-counts}.
Moreover, \(B_I\le M_I\), and for every level \(0\le j<K-1\),
\[
Q_j(I)\le \min\left\{8M_I B_I,\frac{|I|}{m_j}\right\}.
\]
\end{lemma}
\begin{proof}
We first prove that \(B_I \leq M_I\).
Fix a group, a source prediction, and a positive level. Each completed boundary block has a comparator that was unavailable at its start, is stored by its completion, and is available on every subsequent round of the same depth interval. Since blocks with these fixed indices are processed consecutively, this comparator is available at the start of every later block and cannot serve as the comparator of a later boundary block. Thus these boundary blocks have distinct comparators, all counted by \(M_I\), so their number is at most \(M_I\). Taking the maximum over groups, sources, and positive levels, and using \(M_I\ge1\), gives \(B_I\le M_I\).

Next, we bound \(Q_j(I)\) for \(0\le j<K-1\).
Fix a group and a source prediction.
Every selected completed level-\(j\) piece has a unique level-\((j+1)\) parent, which is either boundary or unfinished.
For this group and source, there are at most \(B_I\) completed boundary parents at level \(j+1\), each with four completed children, and at most one unfinished parent, with at most three completed children.
Thus there are at most \(4B_I+3\) selected completed level-\(j\) pieces for this group and source.
Summing over at most \(M_I\) sources and then taking the maximum over groups gives
\[
Q_j(I)\le(4B_I+3)M_I\le8M_IB_I\le8M_I^2.
\]
For each group, the selected completed level-\(j\) pieces have total weight at most \(\sum_{t\in I}g(x_t)\sum_a\pi_{t,a}\le|I|\), and each has weight at least \(m_j\).
Taking the maximum of their counts over groups gives \(Q_j(I)\le|I|/m_j\).

\end{proof}

\subsection{Proof of the residual-piece bound}

\begin{lemma}[Calibration error from residual-piece counts]\label{lem:piece-count-residual}
On \(\mathcal E_{\rm global}\), for every depth \(K\) and every nonempty prefix \(I\) of a depth-\(K\) interval, as defined in \Cref{def:discovery-counts}, we have
\[
\begin{aligned}
\MCerr_{\calG}^{(1)}(I)\le{}&\underbrace{6\sqrt\kappa\,|I|w_{K-1}}_{\substack{\text{completed pieces}\\\text{at scale }K-1}}
+\underbrace{12288\sqrt\kappa\,nL\sum_{j=0}^{K-2}\frac{Q_j(I)}{w_j}}_{\substack{\text{completed pieces}\\\text{below scale }K-1}}
\\
&+\underbrace{2\min(|I|,m_0M_I)}_{\substack{\text{unfinished}\\\text{scale-zero pieces}}}
+4\sqrt{nLM_IT}+2\beta|I|.
\end{aligned}
\]
\end{lemma}

\begin{proof}[Proof of \Cref{lem:piece-count-residual}]
On \(\mathcal E_{\rm global}\), the sampling bound \eqref{eq:iclr-prediction-sampling}, the smoothing correction \(\|y_t-y'_t\|_1\le2\beta\), and the block partition defined in Section~\ref{def:boundary-analysis} give
\[
\MCerr_{\calG}^{(1)}(I)
\le\max_{g\in\calG}\sum_a\sum_{E\in\mathcal P_{g,a}(I)}
N_E\|z_E-a\|_1
+4\sqrt{nLM_IT}+2\beta|I|.
\]
Fix a group \(g\).
Here the sampling bound uses the \(M_I\) distinct candidate predictions counted at the end of \(I\), which include every prediction value used in \(I\).
The partition includes unfinished scale-zero pieces of positive weight, so their means are also well-defined.
We bound the piece costs separately for completed pieces at the top scale, completed pieces at lower scales, and unfinished scale-zero pieces.

For each selected completed scale-\(j\) residual piece, \Cref{lem:local-comparison} gives \(N_E\|z_E-a\|_1\le6\sqrt\kappa\,w_jN_E\).
The completed scale-\((K-1)\) residual pieces have total weight at most \(|I|\), so their total cost is bounded by
\[
\sum_a
\sum_{\substack{E\in\mathcal P_{g,a}(I)\\
                E\text{ completed at scale }K-1}}
N_E\|z_E-a\|_1
\le6\sqrt\kappa\,w_{K-1}|I|.
\]
Now fix a lower scale \(0\le j<K-1\).
Each completed scale-\(j\) piece has \(N_E<2m_j\), and there are at most \(Q_j(I)\) such selected pieces over all sources for this group.
Using \(m_j=1024nL/w_j^2\), their total cost therefore satisfies
\[
\begin{aligned}
\sum_a
\sum_{\substack{E\in\mathcal P_{g,a}(I)\\
                E\text{ completed at scale }j}}
N_E\|z_E-a\|_1
&\le12\sqrt\kappa\,w_jm_jQ_j(I)\\
&=\frac{12288\sqrt\kappa\,nLQ_j(I)}{w_j}.
\end{aligned}
\]
There is at most one unfinished scale-zero piece per source, each with \(N_E<m_0\).
Since \(\|z_E-a\|_1\le\|z_E\|_1+\|a\|_1=2\), each such piece has cost at most \(2N_E\).
Their total weight is also at most \(|I|\), by disjointness and \(\sum_a\pi_{t,a}=1\).
There are at most \(M_I\) sources, so
\[
\sum_a
\sum_{\substack{E\in\mathcal P_{g,a}(I)\\
                E\text{ unfinished at scale }0}}
N_E\|z_E-a\|_1
\le2\min(|I|,m_0M_I).
\]
Summing over the lower scales, taking the maximum over groups, and substituting the three cost bounds into the opening inequality gives
\[
\begin{aligned}
\MCerr_{\calG}^{(1)}(I)\le{}&6\sqrt\kappa\,|I|w_{K-1}
+12288\sqrt\kappa\,nL\sum_{j=0}^{K-2}\frac{Q_j(I)}{w_j}
+2\min(|I|,m_0M_I)\\
&+4\sqrt{nLM_IT}+2\beta|I|.
\end{aligned}
\]

The event is uniform over intervals and prediction subsets, so the same argument holds simultaneously at every permitted endpoint.
\end{proof}

\subsection{Proof of the discovery-count corollary}
\label{subsec:general-packing-proof}

\begin{proof}[Proof of \Cref{cor:discovery-residual}]
On \(\mathcal E_{\rm global}\), start from \Cref{lem:piece-count-residual}.
Write \(w=w_{K-1}\), so \(w\le1\) because \(K\ge2\).

For the lower-scale term, the proof of \Cref{lem:residual-charging} gives \(Q_j(I)\le8M_IB_I\).
Since \(w_j=2^{1-j}\), we have \(\sum_{j=0}^{K-2}w_j^{-1}=w^{-1}-1/2<w^{-1}\), and hence
\[
\begin{aligned}
12288\sqrt\kappa\,nL\sum_{j=0}^{K-2}\frac{Q_j(I)}{w_j}
&\le12288\sqrt\kappa\,nL\cdot8M_IB_I
\sum_{j=0}^{K-2}\frac1{w_j}\\
&\le98304\sqrt\kappa\,\frac{nLM_IB_I}{w}.
\end{aligned}
\]
For the unfinished scale-zero term, \(m_0=1024nL/w_0^2=256nL\) gives
\[
\begin{aligned}
2\min(|I|,m_0M_I)
&\le2m_0M_I\\
&=512nLM_I\\
&\le512\sqrt\kappa\,\frac{nLM_IB_I}{w},
\end{aligned}
\]
where the last inequality uses \(w\le1\) and \(B_I,\kappa\ge1\).
Substituting these two bounds into \Cref{lem:piece-count-residual} and using \(98304+512=98816\) gives
\[
\MCerr_{\calG}^{(1)}(I)\le6\sqrt\kappa\,|I|w
+98816\sqrt\kappa\,\frac{nLM_IB_I}{w}
+4\sqrt{nLM_IT}+2\beta|I|.
\]
Bounding \(6\) and \(98816\) by \(10^6\) in this estimate proves the corollary.
\end{proof}

\section{Choosing the resolution from the data}

\subsection{Adaptive procedure}
\label{app:adaptive-procedure}
\Cref{alg:adaptive-depth} implements Section~\ref{sec:iclr-controller}, with \(\Gamma_K\) as defined there.
The interval residual includes all rounds through the test, including weight in unfinished blocks.
For its incremental computation, maintain
\[
Y_{I,g,a}=\sum_{t\in I}g(x_t)\one\{p_t=a\}(y_t-a),
\qquad \MCerr_{\calG}^{(1)}(I)=\max_{g\in\calG}\sum_a\|Y_{I,g,a}\|_1.
\]

\begin{algorithm}[H]
\footnotesize
\caption{Choosing the number of accuracy levels}
\label{alg:adaptive-depth}
\begin{algorithmic}
\Require Inputs of \Cref{alg:fixed-depth}, except the prescribed depth \(K\).
\State Set \(K\gets K_{\max}\) and initialize the state of \Cref{alg:fixed-depth} once.
\State Set \(I\gets\varnothing\) and clear all calibration-residual
accumulators. Any previously uninitialized \(Y_{I,g,a}\) is
interpreted as the zero vector in \(\mathbb R^n\).
\For{\(t=1,\ldots,T\)}
    \State Execute round \(t\) of \Cref{alg:fixed-depth} at depth \(K\), including its comparison update and all discovery updates.
    \State Extend \(I\gets I\cup\{t\}\), carrying forward its residual vectors.
    \ForAll{\(g\in\calG\)}
        \State Set \(Y_{I,g,p_t}\gets Y_{I,g,p_t}+g(x_t)(y_t-p_t)\), leaving other prediction entries unchanged.
    \EndFor
    \State Compute \(\MCerr_{\calG}^{(1)}(I)=\max_{g\in\calG}\sum_a\|Y_{I,g,a}\|_1\), summing over all distinct candidate predictions discovered so far.
    \If{\(\MCerr_{\calG}^{(1)}(I)>\Gamma_K\), \(K>1\), and \(t<T\)}
        \State End \(I\) at \(t\) and set \(K\gets K-1\).
        \State Clear all local block accumulators and residual vectors, and set \(I\gets\varnothing\).
        \State Retain all candidate prediction sets and comparison aggregates whose endpoints remain available.
    \EndIf
\EndFor
\end{algorithmic}
\end{algorithm}

\subsection{Proof of the affine prediction and boundary-block counts}

\begin{proof}[Proof of \Cref{lem:iclr-affine-counts}]
Fix an interval \(I\) at depth \(K\), and abbreviate \(P_j(\max I+1)\) by \(P_j\) for \(j\ge1\) throughout this proof.
Fix \(w=w_{K-1}\ge64\sqrt{\rho_\infty(S)}\) and let \(S_\beta=(1-\beta)S+\beta q_0\).
Its affine dimension is at most \(s\).
The packing bounds below therefore hold with exponent \(s\), since their bases are at least one.
Smoothing multiplies \(\ell_1\) distances by \(1-\beta\), so \(H^2\le\|\cdot\|_1\) gives \(\dist_H((1-\beta)\mu_t+\beta q_0,S_\beta)\le\sqrt{\rho_\infty(S)}\le w/64\).
By convexity of \(S_\beta\) and joint convexity of \(H^2\), the same bound holds for every weighted predictable mean.
Each candidate prediction in \(P_j\) comes from a completed scale-\(\ell\) block \(E\) with \(\ell\ge j\).
Applying \eqref{eq:completed-block-concentration} at scale \(\ell\) and using the triangle inequality gives
\[
\dist_H(z_E,S_\beta)
\le H(z_E,\bar\mu'_E)+\dist_H(\bar\mu'_E,S_\beta)
\le\frac{w_\ell}{16}+\frac{w}{64}
\le\frac{w_j}{8}.
\]
This includes values inherited from earlier depths and those inserted on the final round of \(I\).

\textbf{Step 1. Count candidate predictions.}
Candidate predictions in \(P_j\) are more than \(w_j\) apart, so their representatives in \(S_\beta\) are more than \(w_j/2\) apart.
To see this, for each \(a\in P_j\), choose a representative \(\tilde a\in S_\beta\) with \(H(a,\tilde a)\le w_j/8\).
For distinct \(a,b\in P_j\), the insertion rule gives \(H(a,b)>w_j\), so
\[
H(\tilde a,\tilde b)
\ge H(a,b)-H(a,\tilde a)-H(b,\tilde b)
>\frac{3w_j}{4}>\frac{w_j}{2}.
\]
Thus, these representatives form a \(w_j/2\)-separated subset of \(S_\beta\), and \(|P_j|\le\GlobPack_H(S_\beta;w_j)\).

The global bound in \Cref{lem:h-affine-packing} gives \(|P_j|
\le \left[C_0^2\left(1+\frac{2\sqrt{2\kappa}}{w_j}\right)\right]^s \leq \left(\frac{a\sqrt{\kappa}}{w_j}\right)^s \)  for a universal \(a=C_0^2(1+2\sqrt2)\ge 1\).
Since \(w_j\ge w\),
\[
M_I=\left|P_0\cup\bigcup_{j=1}^{K-1}P_j\right| \le1+\sum_{j=1}^{K-1}|P_j|
\le K_{\max} (a\sqrt\kappa/w)^s.
\]

\textbf{Step 2. Count boundary comparators in chronological order.}
Recall that \(B_I\) is the maximum, over groups, sources, and positive levels, of the number of completed boundary blocks in \(I\), with a minimum value of \(1\).
For each fixed group, source, and level, these blocks have distinct comparators: each earlier comparator remains available throughout every later block and therefore cannot serve as the comparator of a later boundary block.
It therefore suffices to bound the number of distinct comparators for each fixed group, source, and positive level.
Fix a group, source \(a\), and positive level \(j\), and list the completed boundary blocks' distinct comparators \(q_1,q_2,\ldots\in P_j\) in closure order.
Let \(E_m\) be the block associated with \(q_m\), with mean \(z_m\) and weight \(N_m\).
For \(i<m\), the earlier comparator \(q_i\) is available throughout \(E_m\), so applying \eqref{eq:iclr-local-hellinger} with source \(a\) and comparator \(q_i\) gives
\[
H(z_m,a)\le\sqrt{3\kappa}\,H(z_m,q_i)
+4\sqrt{\frac{\kappa L}{N_m}}.
\]
By the definition of \(q_m\), we have \(H(q_m,z_m)\le w_j\).
Applying the triangle inequality twice therefore yields
\[
\begin{aligned}
H(q_m,a)
&\le H(q_m,z_m)+H(z_m,a)\\
&\le H(q_m,z_m)
+\sqrt{3\kappa}\bigl(H(z_m,q_m)+H(q_m,q_i)\bigr)
+4\sqrt{\frac{\kappa L}{N_m}}\\
&\le\sqrt{3\kappa}\,H(q_m,q_i)
+(1+\sqrt{3\kappa})w_j
+4\sqrt{\frac{\kappa L}{N_m}}.
\end{aligned}
\]
Since \(E_m\) is completed, \(N_m\ge m_j=1024nL/w_j^2\), and hence
\[
4\sqrt{\frac{\kappa L}{N_m}}
\le\frac{\sqrt\kappa}{8\sqrt n}\,w_j.
\]
Using \(n,\kappa\ge1\), both \(\sqrt{3\kappa}\) and \(1+\sqrt{3\kappa}+\sqrt\kappa/(8\sqrt n)\) are at most \(16\sqrt\kappa\).
Thus
\[
H(q_m,a)\le A\bigl(H(q_m,q_i)+w_j\bigr),
\qquad A=16\sqrt\kappa.
\]
We count these comparators by partitioning them according to their Hellinger distance from the source \(a\).
For each \(q_i\), choose a nearby representative \(\tilde q_i\in S_\beta\) with \(H(q_i,\tilde q_i)\le w_j/8\).
We bound the separation and radius of the representatives in each region:
\begin{itemize}
\item \textbf{Near region: \(H(q_i,a)\le4Aw_j\).}
By the separation argument in Step 1, \(H(\tilde q_i,\tilde q_k) > w_j/2\).
Each representative also satisfies
\[
H(\tilde q_i,a)
\le H(\tilde q_i,q_i)+H(q_i,a)
\le\frac{w_j}{8}+4Aw_j\le5Aw_j.
\]
Thus the representatives are separated by more than \(w_j/2\) and lie in the ball of radius \(5Aw_j\) centered at \(a\).

\item \textbf{Far regions.}
For each integer \(h\ge0\), consider the comparators satisfying \(4A2^hw_j<H(q_i,a)\le8A2^hw_j\).
For two comparators in the same far region, order them as \(i<m\).
The preceding chronological comparison gives
\[
H(q_m,q_i)
\ge\frac{H(q_m,a)}{A}-w_j
>(4\cdot2^h-1)w_j.
\]
Replacing both comparators by their representatives loses at most \(w_j/4\), so
\[
H(\tilde q_m,\tilde q_i)
>(4\cdot2^h-1)w_j-\frac{w_j}{4}
>2\cdot2^hw_j.
\]
Each representative lies within distance
\[
H(\tilde q_i,a)
\le\frac{w_j}{8}+8A2^hw_j
\le9A2^hw_j
\]
of \(a\). Thus, the representatives in the \(h\)-th far region are separated by more than \(2\cdot2^h w_j\) and lie in the ball of radius \(9A2^h w_j\) centered at \(a\).
\end{itemize}
Apply the local packing bound in \Cref{lem:h-affine-packing} with \(v=w_j\) for the near region and \(v=2^hw_j\) for the \(h\)-th far region.
In both cases, the representatives are separated by more than \(v/2\) and lie within radius \(Rv\) of \(a\), where \(R=300\sqrt\kappa>9A\).
The bound therefore limits each region to
\[
[C_0^2(1+1200\sqrt2\,\kappa)]^s\le(b\kappa)^s
\]
comparators, for a universal \(b\ge1\).
It remains to count the nonempty regions.
For any probability vectors \(p,q\), we have \(H(p,q)\le\sqrt2\).
If \(h\ge j\), a comparator in the \(h\)-th far region would have to satisfy
\[
H(q_i,a)>4A2^hw_j
=8A2^{h-j}\ge8A>\sqrt2,
\]
which is impossible.
Thus only the far regions indexed by \(h=0,\ldots,j-1\) can be nonempty.
There are at most \(j\) such regions and one near region, giving at most \(j+1\le K\le K_{\max}\) regions in total.
Each region contains at most \((b\kappa)^s\) comparators, so there are at most \(K_{\max}(b\kappa)^s\) comparators in total.
Since \(K_{\max}(b\kappa)^s\ge1\), this bound also covers the minimum value of \(1\) in the definition of \(B_I\).
By the reduction at the start of this step,
\[
B_I\le K_{\max}(b\kappa)^s.
\]
\end{proof}

\subsection{Proof of Theorem~\ref{thm:uniform-affine}}
\label{app:iclr-affine-upper-proof}

\begin{proof}
Work on \(\mathcal E_{\rm global}\).
We prove the stronger guarantee in \Cref{rem:affine-robustness} and then specialize it to the exact affine dimension of the means.
Fix any nonempty closed convex \(S\subseteq\DeltaD\) and any integer \(s\) with \(\operatorname{affdim}(S)\le s\le n-1\), and write \(\rho=\rho_\infty(S)\).
The algorithm adjusts its current depth \(K\) using only the observed residual tests.
For any visited depth \(K\ge2\) with width \(w=w_{K-1}\ge64\sqrt\rho\), \Cref{lem:iclr-affine-counts} bounds \(M_I\) and \(B_I\) for every nonempty prefix \(I\) of the depth-\(K\) interval.
Using \(|I|\le T\) and substituting these counts into \Cref{cor:discovery-residual} gives
\begin{equation}\label{eq:comparison-width-error}
\begin{aligned}
\MCerr_{\calG}^{(1)}(I)\le10^6\sqrt{\kappa}\biggl[&\underbrace{Tw}_{\text{top-piece cost}}+\underbrace{\frac{nLK_{\max}^2(ab\kappa^{3/2})^s}{w^{s+1}}}_{\text{discovery cost}}\\
&+\underbrace{\sqrt{nLK_{\max}(a\sqrt\kappa/w)^sT}}_{\text{sampling correction}}\biggr]
+\underbrace{2\beta|I|}_{\text{smoothing correction}}.
\end{aligned}
\end{equation}
We first choose a comparison width, then show that this error bound prevents the algorithm from moving to a coarser scale, and finally sum the interval errors.
The comparison width depends on \(s,S\) only in the analysis and is not an input to the algorithm.

\textbf{Step 1. Choose a comparison width.}
We seek a small width \(w\) for which the discovery cost is controlled by the top-piece cost \(Tw\).
Increasing \(w\) decreases the discovery cost but increases \(Tw\).
The width at which these two costs are equal is
\[
Tw_\star^{s+2}=
{nLK_{\max}^2(ab\kappa^{3/2})^s}.
\]
Since \(LK_{\max}^2(ab\kappa^{3/2})^s\ge1\), we have
\begin{equation}\label{eq:lowerboundofw}
w_\star\ge(n/T)^{1/(s+2)}\ge T^{-1/2}.
\end{equation}
To ensure that \Cref{lem:iclr-affine-counts} applies, we also require \(w\ge64\sqrt\rho\). Accordingly, set
\[
v_\star=\max\{2w_\star,64\sqrt\rho\}.
\]
If \(v_\star>1\), the deterministic bound gives \(\MCerr_{\calG}^{(1)}(T)\le2T<2Tv_\star\).
We include this case in the final parameter substitution in Step 3.
For the arguments below, assume \(v_\star\le1\).
The available widths at depths \(2,\ldots,K_{\max}\) are
\[
1,\ \frac12,\ \ldots,\ 2^{-\lceil\log_2T\rceil}.
\]
Since
\[
2^{-\lceil\log_2T\rceil} \le T^{-1} \le 2T^{-1/2} \le 2w_\star \leq v_\star \le1,
\]
\(v_\star\) lies within this range.
Round \(v_\star\) up to the smallest available width that is at least \(v_\star\), and let \(K^\star\) be the depth with its finest width used denoted by \(w_{K^\star-1} \geq v_\star\).
Consecutive widths differ by a factor of two, so
\[
2\le K^\star\le K_{\max},
\qquad
v_\star\le w_{K^\star-1}<2v_\star.
\]
Thus, \(w_{K^\star-1}\ge64\sqrt\rho\), as required by \Cref{lem:iclr-affine-counts}. Moreover, \(w_{K^\star-1}\ge2w_\star\ge T^{-1/2}\) by (\ref{eq:lowerboundofw}), which implies \(Tw_{K^\star-1}\ge1\) since \(n,T\ge1\) and \(s\ge0\).
The discovery term at this width satisfies
\[
\frac{nLK_{\max}^2(ab\kappa^{3/2})^s}{w_{K^\star-1}^{s+1}}
=Tw_{K^\star-1}\left(\frac{w_\star}{w_{K^\star-1}}\right)^{s+2}
\le\frac{Tw_{K^\star-1}}4,
\]
because \(w_{K^\star-1}\ge2w_\star\) and \(s\ge0\).

\textbf{Step 2. The algorithm never moves to a coarser scale.}
If the algorithm reaches depth \(K^\star\), \eqref{eq:comparison-width-error} applies with \(w=w_{K^\star-1}\) to every nonempty prefix \(I\) of the resulting depth-\(K^\star\) interval, since Step 1 ensures \(w_{K^\star-1}\ge64\sqrt\rho\).
Step 1 bounds the discovery term by \(Tw_{K^\star-1}/4\).
Using this bound and \(K_{\max}(b\kappa)^s\ge1\), the sampling correction satisfies
\[
\begin{aligned}
&\sqrt{nLK_{\max}
\left(\frac{a\sqrt\kappa}{w_{K^\star-1}}\right)^sT}\le\sqrt{Tw_{K^\star-1}\,
\frac{nLK_{\max}^2(ab\kappa^{3/2})^s}{w_{K^\star-1}^{s+1}}}
\le\frac{Tw_{K^\star-1}}2.
\end{aligned}
\]
Finally, since \(\beta=1/(2T^2)\) and \(|I|\le T\), the smoothing correction satisfies \(2\beta|I|=|I|/T^2\le1/T\).
Substituting these bounds gives, after every round at this depth,
\[
\begin{aligned}
\MCerr_{\calG}^{(1)}(I)
&\le10^6\sqrt\kappa\left(1+\frac14+\frac12\right)Tw_{K^\star-1}+\frac1T\\
&=\frac74\cdot10^6\sqrt\kappa\,Tw_{K^\star-1}+\frac1T\\
&<3\cdot10^6\sqrt\kappa\,Tw_{K^\star-1}=\Gamma_{K^\star},
\end{aligned}
\]
where the strict inequality uses \(Tw_{K^\star-1}\ge1\) and \(\sqrt{\kappa}\ge1\).
Thus, if the algorithm reaches depth \(K^\star\), it cannot leave it.
Since depth decreases one level at a time from \(K_{\max}\), every visited depth satisfies \(K\ge K^\star\), and every visited width is at most \(w_{K^\star-1}\).
No affine bound is needed at the finer widths visited before this depth.

\textbf{Step 3. Sum the interval errors and obtain the final rate.}
Let \(J_1,\ldots,J_m\) be the visited intervals, with depths \(K_1>\cdots>K_m\) and \(m\le K_{\max}\).
For each group in \(\calG\), one round changes just one residual vector, by \(g(x_t)(y_t-p_t)\), whose norm is at most two.
Hence  \(\MCerr_{\calG}^{(1)}(J_i)\le\Gamma_{K_i}+2\) holds even if the horizon test fails that is when the budget is first exceeded at $t=T$.
The intervals partition all rounds, including failed endpoints.
The triangle inequality across intervals gives
\[
\MCerr_{\calG}^{(1)}(T)
\le\sum_{i=1}^m \MCerr_{\calG}^{(1)}(J_i)
\le3\cdot10^6\sqrt{\kappa}\,T\sum_{i=1}^m w_{K_i-1}+2m
<6\cdot10^6\sqrt{\kappa}\,Tw_{K^\star-1}+2K_{\max}.
\]
The last inequality uses the doubling of successive visited widths and their upper bound \(w_{K^\star-1}\): \(\sum_{i=1}^m w_{K_i-1}
\le w_{K^\star-1}\left(1+\frac12+\cdots+\frac1{2^{m-1}}\right)
<2w_{K^\star-1}\).
Substituting \(w_{K^\star-1}<2v_\star\) from Step 1 gives
\[
\MCerr_{\calG}^{(1)}(T)\le12\cdot10^6\sqrt\kappa\,Tv_\star+2K_{\max}.
\]
For \(v_\star>1\), the deterministic estimate in Step 1 implies the same inequality, without using a comparison depth.
Since \(v_\star\le2w_\star+64\sqrt\rho\), both cases imply
\[
\MCerr_{\calG}^{(1)}(T)
\le
\underbrace{24\cdot10^6\sqrt\kappa\,Tw_\star}_{\text{width term}}
+
\underbrace{768\cdot10^6\sqrt\kappa\,T\sqrt\rho}_{\text{approximation term}}
+
\underbrace{2K_{\max}}_{\text{interval overshoot term}}.
\]
We now bound the three terms on the right-hand side in the same order.
Appendix~\ref{app:universal-constants} bounds \(\kappa,L,K_{\max}\) by \(8,40,6\) times \(\log(en|\barG|(T+1)/\delta)\), respectively.
\begin{itemize}
\item \emph{Width term.}
Substituting the definition of \(w_\star\) gives
\[
\begin{aligned}
\sqrt\kappa\,Tw_\star
&=
\kappa^{1/2+3s/[2(s+2)]}
\bigl[LK_{\max}^2\bigr]^{1/(s+2)}(ab)^{s/(s+2)}
n^{1/(s+2)}T^{(s+1)/(s+2)}\\
&\le384\sqrt{40}\,ab\left[\log\!\left(\frac{en|\barG|(T+1)}{\delta}\right)\right]^2
n^{1/(s+2)}T^{(s+1)/(s+2)}.
\end{aligned}
\]
Here the exponent of the logarithm is
\[
\frac12+\frac{3s/2+3}{s+2}=2.
\]
Since \(s\ge0\) and \(a,b\ge1\), the constant factor is at most \(8^2(40\cdot6^2)^{1/2}ab=384\sqrt{40}\,ab\), independently of \(s\).

\item \emph{Approximation term.}
Using the bound on \(\kappa\) and \(\log(en|\barG|(T+1)/\delta)\ge1\),
\[
\begin{aligned}
\sqrt\kappa\,T\sqrt\rho
&\le\sqrt8\left[\log\!\left(\frac{en|\barG|(T+1)}{\delta}\right)\right]^{1/2}T\sqrt\rho\\
\end{aligned}
\]

\item \emph{Interval overshoot term.}
Using the bound on \(K_{\max}\), \(\log(en|\barG|(T+1)/\delta)\ge1\), and \(n,T\ge1\),
\[
\begin{aligned}
2K_{\max}&\le12\log\!\left(\frac{en|\barG|(T+1)}{\delta}\right) \le12\left[\log\!\left(\frac{en|\barG|(T+1)}{\delta}\right)\right]^2
n^{1/(s+2)}T^{(s+1)/(s+2)}.
\end{aligned}
\]
\end{itemize}
Substituting these three bounds into the error estimate and using
\[
\log\!\left(\frac{en|\barG|(T+1)}{\delta}\right)
\le2\log\!\left(\frac{eTn|\barG|}{\delta}\right),
\]
which follows from \(T+1\le2T\) and \(\log(eTn|\barG|/\delta)\ge1\), gives, for a universal \(C\),
\[
\MCerr_{\calG}^{(1)}(T)
\le C\left[\log\!\left(\frac{eTn|\barG|}{\delta}\right)\right]^2
\left(n^{1/(s+2)}T^{(s+1)/(s+2)}+T\sqrt\rho\right).
\]
Because the event \(\mathcal E_{\rm global}\) does not depend on \(S,s\), this bound holds simultaneously for every admissible \(S,s\), proving \Cref{rem:affine-robustness}.
Taking \(S=\conv\{\mu_1,\ldots,\mu_T\}\) and \(s=\operatorname{affdim}(S)\) gives \(\rho_\infty(S)=0\) and hence \Cref{thm:uniform-affine}, even when the means and their dimension depend on the realized history.
\end{proof}

\subsection{Runtime}

\label{cov:h-runtime}
We count exact real arithmetic operations, comparisons, evaluations of logarithms and square roots, and evaluations of the given group functions, each at unit cost.
A fresh uniform draw from \([0,1]\) also costs one operation; sampling a finite prediction distribution uses its cumulative probabilities.
There are at most \(A_{\max}\) distinct candidate predictions, at most \(|\barG|A_{\max}^2\) pair aggregates, and at most one source/group closure per scale per round.
Each recursive call increases scale, so all closure loops are finite.
Evaluating the losses of available predictions costs \(O(TnA_{\max})\) operations.
Each of at most \(T|\barG|A_{\max}K_{\max}\) closures scans at most \(A_{\max}\) candidate predictions at each of at most \(K_{\max}\) membership levels, for \(O(Tn|\barG|A_{\max}^2K_{\max}^2)\) operations.
Recognizing an already discovered value takes at most \(O(nA_{\max})\) comparisons per closure and fits this bound.
Solving the linear stationarity equations costs \(O(TA_{\max}^3)\), and pair updates cost \(O(T|\barG|A_{\max}^2)\).
Group evaluations, residual updates, block-weight updates and sampling fit these bounds.
Together with the parameter choices, these bounds give a polynomial operation count in \(T\), \(n\), \(|\barG|\), and \(\log(1/\delta)\) under the stated exact-operation model.

\section{Affine-dimensional lower bound}
\label{app:affine-lower}

For the lower bound, we construct a sequence of contexts and specify the outcome distribution for each context.
We associate each context with a point on an \(s\)-dimensional grid in the simplex.
The learner’s observation \(y_t\) is the average of this grid point and an independently and uniformly chosen simplex vertex.
The resulting conditional means lie in an affine space of dimension at most \(s\), while the random observations fluctuate across all \(n\) coordinates.

\begin{proof}[Proof of \Cref{thm:affine-lower}]
The proof separates two cases according to the expected total distance between predictions and conditional means.
If this distance is large, the group family gives a calibration lower bound.
If it is small, the separation between grid means forces the learner to use many distinct prediction values.
We then use the noise accumulated within these prediction bins to obtain the lower bound.
We choose the grid resolution so that both cases yield the claimed rate.
Write \(L_T:=\log(T+1)\) and \(R_s(T,n):=n^{1/(s+2)}T^{(s+1)/(s+2)}\), and assume \(s\ge1\) and \(T\ge\max\{n,2^{s+2}\}\). The proof follows the following five steps.

\par\smallskip\noindent\textbf{Step 1. Construct a low-dimensional grid with noise on all coordinates.}
Let \(k:=\lfloor n/2\rfloor\) and \(m:=2^{\max\{1,\lceil\log_2(T/k)/(s+2)\rceil\}}\).
We construct an oblivious adversary with context space \([m]^s\).
The contexts \(x_t=(x_{t,1},\ldots,x_{t,s})\), for \(t=1,\ldots,T\), cycle through this space in a fixed order.
We associate each context \(x=(x_1,\ldots,x_s)\in[m]^s\) with a probability vector \(q(x)\).
The first \(s\) coordinates of these vectors form a Cartesian grid with \(m\) equally spaced values in \([a,2a]\) along each coordinate.

With \(a:=1/[4(s+1)]\), define
\[
q_j(x):=a+a\frac{x_j-1}{m-1}
\quad (1\le j\le s).
\]
Complete this to an \(n\)-coordinate probability vector by setting
\[
q(x):=
\Bigl(
\underbrace{q_1(x),\ldots,q_s(x)}_{s\text{ varying coordinates}},
\underbrace{1-\sum_{j=1}^s q_j(x)}_{\text{remaining probability mass}},
\underbrace{0,\ldots,0}_{n-s-1\text{ coordinates}}
\Bigr).
\]
The final zero block is absent when \(s=n-1\).
Our choice of \(a\) makes the remaining coordinate at least \(1-2sa=(s+2)/(2(s+1))>0\), so all coordinates are nonnegative and sum to one.

On round \(t\), the adversary selects a simplex vertex \(e_{J_t}\) by drawing its index \(J_t\) uniformly from \([n]\).
These draws are independent across rounds and of the learner's randomness.
Fix an arbitrary forecaster and pre-sample its private randomness independently of \(J_{1:T}\).
Let \(\mathcal H_t\) be the history consisting of that randomness and \(J_1,\ldots,J_{t-1}\).
Then \(p_t\) is \(\mathcal H_t\)-measurable.
Writing \(q_0=(1/n,\ldots,1/n)\), the outcome and its conditional mean are
\begin{equation}
y_t:=\frac{q(x_t)+e_{J_t}}2,\qquad
\mu_t
=\mathbb E[y_t\mid\mathcal H_t]
=\frac{q(x_t)+\mathbb E[e_{J_t}\mid\mathcal H_t]}2
=\frac{q(x_t)+q_0}2.
\label{eq:unified-lower-outcomes}
\end{equation}
For \(j\le s\), this gives
\[
\mu_{t,j}
=\frac{q_j(x_t)+1/n}{2}
=\frac12\left(a+a\frac{x_{t,j}-1}{m-1}+\frac1n\right).
\]
The grid \(q([m]^s)\) lies in an affine space of dimension \(s\).
Scaling by \(1/2\) and translating by \(q_0/2\) preserves this dimension, so the visited conditional means have affine dimension at most \(s\).

Then, we construct binary group functions on the context space.
For \(j\le s\), we compare \(p_{t,j}\), the \(j\)-th entry of the prediction \(p_t\), with \(\mu_{t,j}\).
To approximate \(\operatorname{sign}(p_{t,j}-\mu_{t,j})\), we use the length-\(m\) Walsh system, which requires \(m\) to be a power of two.
Let \(\psi_\ell(r)=(-1)^{\langle\ell,r\rangle_2}\), where \(\ell,r\in\{0,\ldots,m-1\}\) and the inner product is taken modulo two between their \(\log_2m\)-bit binary expansions.

Lemma~11 of \citet{collina2026multicalibration}, with approximation error \(1/2\), gives a fixed set \(H_m\subseteq\{1,\ldots,m-1\}\) with \(|H_m|\le4C_{\rm sub}(\log m)^3\), where \(C_{\rm sub}\) is a universal constant.
The lemma also provides functions \(b_0,b_\ell:[0,1]\to\mathbb R\), \(\ell\in H_m\), with \(\|b_0\|_\infty\le1\) and \(\sum_{\ell\in H_m}\|b_\ell\|_\infty\le\log_2m\), fixed before play.
Here the supremum norms are over \([0,1]\), and we use \(\operatorname{sign}(u)=1\) for \(u\ge0\) and \(-1\) otherwise.
To approximate \(\operatorname{sign}(p_{t,j}-\mu_{t,j})\), we weight each Walsh function \(\psi_\ell(x_{t,j}-1)\) by \(b_\ell(p_{t,j})\) and add the constant term \(b_0(p_{t,j})\).
For every \(j\le s\), round \(t\), and possible prediction \(p_t\in\DeltaD\), this approximation satisfies
\begin{equation}
\left|b_0(p_{t,j})+\sum_{\ell\in H_m}b_\ell(p_{t,j})\psi_\ell(x_{t,j}-1)
-\operatorname{sign}(p_{t,j}-\mu_{t,j})\right|\le\tfrac12.
\label{eq:walsh-threshold-approximation}
\end{equation}

For each \(j\le s\) and \(\ell\in H_m\), split the contexts into two groups according to whether \(\psi_\ell(x_j-1)\) equals \(+1\) or \(-1\).
Their indicator functions are
\[
g_{j,\ell}^{\pm}(x):=\frac{1\pm\psi_\ell(x_j-1)}2
\qquad(j\le s,\ \ell\in H_m).
\]
Let \(\calG_{T,s}\) consist of these functions and the constant function \(g(x)\equiv1\).
Since
\[
\psi_\ell(x_j-1)=g_{j,\ell}^{+}(x)-g_{j,\ell}^{-}(x),
\]
the approximation in \eqref{eq:walsh-threshold-approximation} can be expressed as a linear combination of these group functions.
The groups depend only on the context, while the coefficients \(b_\ell(p_{t,j})\) depend on the prediction.
The bound on \(|H_m|\) and the choice of \(m\) give
\[
|\calG_{T,s}|\le1+2s|H_m|\le C L_T^3,
\]
where \(C=1+64C_{\rm sub}\).
This completes the construction.

\par\smallskip\noindent\textbf{Step 2. Compare predictions with conditional means.}
We establish a distance comparison and a prediction-count bound, which Step~4 combines with the noise estimate from Step~3.

For distinct contexts \(x_t,x_{t'}\), some coordinate \(j\le s\) differs, and hence
\[
\begin{aligned}
\|\mu_t-\mu_{t'}\|_1
&\ge|\mu_{t,j}-\mu_{t',j}|\\
&=\frac{a}{2(m-1)}|x_{t,j}-x_{t',j}|
\ge\frac{a}{2(m-1)}=:\Delta.
\end{aligned}
\]
We next relate the expected total prediction distance \(\E[\sum_{t=1}^T\|p_t-\mu_t\|_1]\) to calibration error.
Write
\[
R:=k^{1/(s+2)}T^{(s+1)/(s+2)}.
\]
We prove the following two claims.
\begin{itemize}
\item \textbf{Large prediction distance forces calibration error.}
We claim that
\begin{equation}
\E\!\left[\sum_{t\le T}\|p_t-\mu_t\|_1\right]
\le B\,\E[\MCerr_{\calG_{T,s}}^{(1)}(T)],
\label{eq:unified-lower-distance}
\end{equation}
where \(B:=4s+8s\log_2m+2\le32L_T\).
In particular, if the expected total prediction distance is at least \(T\Delta/8\), then the expected calibration error is at least \(T\Delta/(8B)\).

To prove \eqref{eq:unified-lower-distance}, we first use the Walsh approximation to bound the prediction error in each of the first \(s\) coordinates.
Fix \(j\le s\).
By \eqref{eq:walsh-threshold-approximation},
\[
\left(b_0(p_{t,j})+\sum_{\ell\in H_m}b_\ell(p_{t,j})\psi_\ell(x_{t,j}-1)\right)
(p_{t,j}-\mu_{t,j})
\ge\frac{|p_{t,j}-\mu_{t,j}|}{2}.
\]

To relate this bound to calibration error, we next replace \(\mu_{t,j}\) by \(y_{t,j}\) in expectation.
For the fixed coordinate \(j\), the conditional mean \(\mu_{t,j}\) has \(m\) possible values.
These values divide \([0,1]\) into \(m+1\) intervals.
For each possible value of \(\mu_{t,j}\), the sign of \(p_{t,j}-\mu_{t,j}\) depends only on which interval contains \(p_{t,j}\).
We include each possible mean value as the lower endpoint of its interval, consistently with \(\operatorname{sign}(0)=1\).

We therefore fix \(b_0\) and \(b_\ell\), \(\ell\in H_m\), to be constant on these intervals, making them measurable functions of the prediction coordinate.
Since \(p_{t,j}\) is \(\mathcal H_t\)-measurable and \(x_t\) is deterministic, the multiplier is \(\mathcal H_t\)-measurable. Then, we have
\begin{equation}\label{eq:relaceubyy}
\begin{aligned}
\E\!\left[
b_\ell(p_{t,j})\psi_\ell(x_{t,j}-1)(p_{t,j}-y_{t,j})
\,\middle|\,\mathcal H_t
\right]&=
b_\ell(p_{t,j})\psi_\ell(x_{t,j}-1)
\bigl(p_{t,j}-\E[y_{t,j}\mid\mathcal H_t]\bigr)\\
&=
b_\ell(p_{t,j})\psi_\ell(x_{t,j}-1)(p_{t,j}-\mu_{t,j}),
\end{aligned}
\end{equation}
where we use \(\E[y_{t,j}\mid\mathcal H_t]=\mu_{t,j}\) to replace \(\mu_{t,j}\) by \(y_{t,j}\) without changing the expected product.

After this roundwise replacement, group together the rounds with the same prediction vector \(v\).
On these rounds, \(p_{t,j}=v_j\), so the coefficients can be factored out as \(b_0(v_j)\) and \(b_\ell(v_j)\).
For each \(\ell\), the identity \(\psi_\ell(x_{t,j}-1)=g^+_{j,\ell}(x_t)-g^-_{j,\ell}(x_t)\) bounds its contribution as follows:
\[
\begin{aligned}
\left|\sum_t b_\ell(p_{t,j})
\psi_\ell(x_{t,j}-1)(p_{t,j}-y_{t,j})\right|
&=\left|\sum_v b_\ell(v_j)
\sum_{t:p_t=v}
\bigl(g^+_{j,\ell}(x_t)-g^-_{j,\ell}(x_t)\bigr)
(v_j-y_{t,j})\right|\\
&\le \|b_\ell\|_\infty
\sum_{\sigma\in\{+,-\}}\sum_v
\left|\sum_{t:p_t=v}
g^\sigma_{j,\ell}(x_t)(v_j-y_{t,j})\right|\\
&\le \|b_\ell\|_\infty
\sum_{\sigma\in\{+,-\}}\sum_v
\left\|\sum_{t:p_t=v}
g^\sigma_{j,\ell}(x_t)(y_t-v)\right\|_1\\
&\le 2\|b_\ell\|_\infty
\operatorname{MCerr}_{\mathcal G_{T,s}}^{(1)}(T).
\end{aligned}
\]
Therefore, together with (\ref{eq:relaceubyy}), we obtain
\[
\begin{aligned}
\frac12\E\!\left[\sum_t|p_{t,j}-\mu_{t,j}|\right]
&\le
\E\!\left[\sum_t b_0(p_{t,j})(p_{t,j}-y_{t,j})\right]\\
&\quad+
\sum_{\ell\in H_m}
\E\!\left[\sum_t
b_\ell(p_{t,j})\psi_\ell(x_{t,j}-1)
(p_{t,j}-y_{t,j})\right]\\
&\le
\left(1+2\sum_{\ell\in H_m}\|b_\ell\|_\infty\right)
\E\!\left[\operatorname{MCerr}_{\mathcal G_{T,s}}^{(1)}(T)\right]\\
&\le
(1+2\log_2m)
\E\!\left[\operatorname{MCerr}_{\mathcal G_{T,s}}^{(1)}(T)\right].
\end{aligned}
\]
For all \(h>s+1\), the mean coordinate is fixed at \(\mu_h=1/(2n)\).
Predictable signs and the constant group control these coordinates together:
\[
\begin{aligned}
\E\!\left[\sum_t\sum_{h>s+1}|p_{t,h}-\mu_h|\right]
&=\E\!\left[\sum_v\sum_{h>s+1}\operatorname{sign}(v_h-\mu_h)
\sum_{t:p_t=v}(v_h-y_{t,h})\right]\\
&\le\E[\MCerr_{\calG_{T,s}}^{(1)}(T)].
\end{aligned}
\]
The bound remains valid when this coordinate set is empty.
For \(z=p_t-\mu_t\), the identity \(\sum_hz_h=0\) implies
\(\|z\|_1\le2\sum_{j\le s}|z_j|+2\sum_{h>s+1}|z_h|\).
Combining these coordinate bounds proves \eqref{eq:unified-lower-distance} with \(B=4s+8s\log_2m+2\).
To check the constants, note that the choice of \(m\) and \(T\ge2^{s+2}\) imply
\[
\log_2m\le\frac{\log_2T}{s+2}+1
\le\frac{2\log_2T}{s+2}.
\]
Writing \(t_0=\log_2T\ge s+2\), we obtain
\[
B\le12s+8t_0+2\le20t_0-22<32L_T,\qquad
\]
using \(20/\log2<32\) and \(L_T>1\). Thus \(1+B\le33L_T\).
The expectations above were taken before grouping, not conditional on the terminal bins.

\item \textbf{Small prediction distance forces many prediction values.}
For each realized prediction value \(v\), let \(n_v:=|\{t:p_t=v\}|\) denote the number of rounds on which the learner predicts \(v\).
Step~3 will lower-bound the expected calibration error in terms of \(\E[Q]\), where
\[
Q:=\sum_v\min\{n_v,\sqrt{k n_v}\}.
\]
We show that
\begin{equation}
\E\!\left[\sum_{t\le T}\|p_t-\mu_t\|_1\right]<T\Delta/8
\quad\Longrightarrow\quad
\E[Q]\ge\frac{R}{2\sqrt2}.
\label{eq:unified-lower-diversity}
\end{equation}
To prove this, we use the number of visits to each context.
There are \(P:=\min\{m^s,T\}\) distinct visited contexts, each occurring at most \(\bar q:=\lceil T/m^s\rceil\le2T/P\) times.
When \(m^s>T\), every visited context occurs once.
Let \(n_v^{\rm near}\) count visits with \(p_t=v\) and \(\|v-\mu_t\|_1<\Delta/4\), and let \(N_{\rm near}=\sum_vn_v^{\rm near}\).
Every other visit has \(\|p_t-\mu_t\|_1\ge\Delta/4\), so
\[
N_{\rm near}\ge T-\frac4\Delta\sum_{t\le T}\|p_t-\mu_t\|_1.
\]
Two distinct grid means cannot both be within \(\Delta/4\) of \(v\).
Hence \(n_v^{\rm near}\le\bar q\), although no such bound is asserted for the total count \(n_v\).
The function \(z\mapsto\min\{z,\sqrt{kz}\}\) is nondecreasing and is at least \(\min\{1,\sqrt{k/\bar q}\}z\) for \(0\le z\le\bar q\).
Thus
\[
Q\ge\min\{1,\sqrt{k/\bar q}\}\,N_{\rm near}.
\]
If \(\E[\sum_{t\le T}\|p_t-\mu_t\|_1]<T\Delta/8\), then \(\E[N_{\rm near}]\ge T/2\), and
\[
\E[Q]\ge\frac T2\min\{1,\sqrt{k/\bar q}\}
\ge\frac1{2\sqrt2}\min\{T,\sqrt{kTP}\}
\ge\frac R{2\sqrt2}.
\]
For the last inequality, \(R\le T\).
If \(P=T\), the displayed minimum equals \(T\).
Otherwise \(P=m^s\), and the upward rounding of \(m\) gives
\(\sqrt{kT}\,m^{s/2}\ge\sqrt{kT}(T/k)^{s/[2(s+2)]}=R\).
This proves \eqref{eq:unified-lower-diversity}, including when only part of the grid is visited.
\end{itemize}
The next step converts this lower bound on \(\E[Q]\) into a calibration lower bound in the small-distance case.

\par\smallskip\noindent\textbf{Step 3. Accumulated noise yields a calibration lower bound.}
For each prediction value \(v\), define its accumulated outcome noise as \(U_v:=\sum_{t:p_t=v}(\mu_t-y_t)\).
Since
\[
\sum_{t:p_t=v}(p_t-y_t)
=\sum_{t:p_t=v}(p_t-\mu_t)+U_v,
\]
the constant group gives the pathwise comparison
\[
\MCerr_{\calG_{T,s}}^{(1)}(T)
\ge\sum_v\|U_v\|_1-\sum_{t\le T}\|p_t-\mu_t\|_1.
\]
Taking expectations and using the distance bound \eqref{eq:unified-lower-distance} from Step~2 yields
\[
(1+B)\,\E[\MCerr_{\calG_{T,s}}^{(1)}(T)]
\ge\E\!\left[\sum_v\|U_v\|_1\right].
\]
Let \(b:=\min\{C_{\rm rev},1/16\}>0\), where \(C_{\rm rev}\) is the universal constant in Theorem~4 of \citet{collina2026multicalibration}.
In the small-distance case, Step~2 gives a lower bound on \(\E[Q]\).
To convert it into a calibration lower bound, we now show that
\[
\E\!\left[\sum_v\|U_v\|_1\right]
\ge \frac{b}{16L_T}\E[Q].
\]
By \eqref{eq:unified-lower-outcomes},
\[
\mu_t-y_t=\frac{q_0-e_{J_t}}2.
\]
The grid-dependent term \(q(x_t)\) cancels, so we can analyze the noise on all \(n\) coordinates, including those used to construct the grid.
To apply the scalar adaptive noise-bucketing bound, use all \(k\) disjoint coordinate pairs \((2j-1,2j)\), \(j\in[k]\).
Since \((\mu_t-y_t)_i=(1/n-\one\{J_t=i\})/2\), each fixed pair satisfies
\[
(\mu_t-y_t)_{2j-1}-(\mu_t-y_t)_{2j}
=\frac{\one\{J_t=2j\}-\one\{J_t=2j-1\}}2.
\]
This difference is zero unless \(J_t\in\{2j-1,2j\}\), and conditional on selecting this pair it is equally likely to be \(-1/2\) or \(1/2\).
Let \(K_t\) denote the index of the coordinate pair selected by \(J_t\): \(K_t=j\) if \(J_t\in\{2j-1,2j\}\), and \(K_t=0\) for the unpaired coordinate when \(n\) is odd.
For each pair \(j\), enumerate the rounds with \(K_t=j\) as \(\tau_{j,1}<\cdots<\tau_{j,N_j}\), where \(N_j\) is the number of such rounds.
Define the noise increments on these rounds by
\[
Z_{j,r}:=(\mu_{\tau_{j,r}}-y_{\tau_{j,r}})_{2j-1}
-(\mu_{\tau_{j,r}}-y_{\tau_{j,r}})_{2j}
=\frac{\one\{J_{\tau_{j,r}}=2j\}-\one\{J_{\tau_{j,r}}=2j-1\}}2.
\]
Each increment takes values in \(\{-1/2,1/2\}\), and its label is the prediction \(p_{\tau_{j,r}}\).
Thus
\[
U_{v,2j-1}-U_{v,2j}
=\sum_{r:p_{\tau_{j,r}}=v}Z_{j,r}.
\]
Because the pairs are disjoint, we have
\[
\begin{aligned}
\sum_{j=1}^{k}|U_{v,2j-1}-U_{v,2j}|
&\le \sum_{j=1}^{k}\bigl(|U_{v,2j-1}|+|U_{v,2j}|\bigr)\\
&=\sum_{i=1}^{2k}|U_{v,i}|
\le \|U_v\|_1.
\end{aligned}
\]
The noise inequality involves square roots of the numbers of visits to each prediction value and coordinate pair, so define
\[
n_{v,j}:=|\{t:p_t=v,\ K_t=j\}|,\qquad
D:=\sum_{v,j}\sqrt{n_{v,j}}.
\]
Applying Theorem~4 of \citet{collina2026multicalibration} to each sequence \((Z_{j,r})_{r=1}^{N_j}\) at amplitude \(h=1/2\), conditional on \(N_j\), and then averaging and summing over pairs gives
\[
\E\!\left[\sum_v\|U_v\|_1\right]
\ge\frac{b}{2L_T}\E[D].
\]
The hypotheses of the cited theorem and the cases \(N_j\le1\) are verified in \hyperref[proof:lower-step3b]{Details of Step~3}.
To relate this bound to \(Q\), note that the current pair is independent of the prediction, with \(\Pr(K_t=j\mid\mathcal H_t)=2/n\).
Let \(D_t\) be the sum defining \(D\) with counts through time \(t\).
Before round \(t\), the prediction value \(p_t\) has appeared \(z\) times, including \(c_j\) rounds with \(K_u=j\).
Then \(\sum_jc_j\le z\), allowing for the unpaired coordinate.
The fresh pair law and the identity \(\sqrt{c+1}-\sqrt c=1/(\sqrt{c+1}+\sqrt c)\) give
\[
\begin{aligned}
\E[D_t-D_{t-1}\mid\mathcal H_t]
&=\frac2n\sum_{j=1}^k(\sqrt{c_j+1}-\sqrt{c_j})\\
&\ge\frac1n\sum_{j=1}^k\frac1{\sqrt{c_j+1}}
\ge\frac{k}{n}\sqrt{\frac{k}{z+k}}.
\end{aligned}
\]
The last inequality uses convexity and monotonicity of \(x\mapsto(1+x)^{-1/2}\).
For a bin with \(n_v\) total visits, its previous-visit counts run through \(0,\ldots,n_v-1\).
Consequently
\[
\sum_{z=0}^{n_v-1}\sqrt{\frac{k}{z+k}}
\ge n_v\sqrt{\frac{k}{n_v+k}}
\ge\frac1{\sqrt2}\min\{n_v,\sqrt{k n_v}\}.
\]
The last bound follows by considering \(n_v\le k\) and \(n_v\ge k\).
Taking expectations, summing the conditional increments over \(t\), and regrouping these chronological visits by prediction value gives
\[
\E[D]\ge\frac{k}{n\sqrt2}\E[Q].
\]
Combining these two bounds with \(k\ge n/3\) gives
\[
\E\!\left[\sum_v\|U_v\|_1\right]
\ge\frac{b}{6\sqrt2\,L_T}\E[Q]
\ge\frac{b}{16L_T}\E[Q].
\]
Substituting into the opening comparison yields
\[
(1+B)\,\E[\MCerr_{\calG_{T,s}}^{(1)}(T)]
\ge\frac{b}{16L_T}\E[Q].
\]
Since \(1+B\le33L_T\), this proves
\begin{equation}
\E[\MCerr_{\calG_{T,s}}^{(1)}(T)]
\ge\frac{b}{528L_T^2}\E[Q].
\label{eq:unified-lower-noise}
\end{equation}

\par\smallskip\noindent\textbf{Step 4. Combine the distance and noise bounds.}
Note that \(L_T\ge(s+1)\log2\). If \(\E[\sum_{t\le T}\|p_t-\mu_t\|_1]\ge T\Delta/8\), then \eqref{eq:unified-lower-distance} and the spacing bound give
\[
\E[\MCerr_{\calG_{T,s}}^{(1)}(T)]
\ge\frac{T\Delta}{256L_T}
\ge\frac{R}{4096(s+1)L_T}
\ge\frac{(\log2)R}{4096L_T^2},
\]
where the second inequality follows that the choice of \(m\) gives \(m\le2(T/k)^{1/(s+2)}\) and hence \(T\Delta\ge R/[16(s+1)]\).
In the small-distance case, \eqref{eq:unified-lower-diversity} and \eqref{eq:unified-lower-noise} give
\[
\E[\MCerr_{\calG_{T,s}}^{(1)}(T)]
\ge\frac{bR}{1056\sqrt2\,L_T^2}.
\]
Both cases imply \(\E[\MCerr_{\calG_{T,s}}^{(1)}(T)]\ge bR/(4096L_T^2)\).
Since \(k\ge n/3\ge n/4\) and \(4^{1/(s+2)}\le2\), we have \(R\ge R_s(T,n)/2\).
Consequently, for every \(s\ge1\) and \(T\ge\max\{n,2^{s+2}\}\),
\begin{equation}
\E[\MCerr_{\calG_{T,s}}^{(1)}(T)]
\ge\frac b{8192}\frac{R_s(T,n)}{L_T^2}.
\label{eq:affine-lower-long-horizon}
\end{equation}

\par\smallskip\noindent\textbf{Step 5. Include the endpoint cases and complete the theorem.}
First take a constant context, the constant group, and independent uniform categorical outcomes \(y_t=e_{J_t}\).
Their conditional mean is always \(q_0\), so its affine dimension is zero.
\hyperref[proof:lower-step5]{Details of Step~5} show that, for every \(T\ge n\),
\begin{equation}
\E[\MCerr_{\{\one\}}^{(1)}(T)]
\ge\frac{\sqrt{nT}}{2\sqrt{10}}.
\label{eq:unified-lower-endpoint}
\end{equation}
This proves the stronger categorical conclusion at \(s=0\).
This endpoint uses only the uniform categorical component of the construction.

For \(s\ge1\) and \(T\ge8\), set
\(r:=\min\{s,\lfloor\log_2T\rfloor-2\}\).
Then \(1\le r\le s\le n-1\) and \(T\ge2^{r+2}\), so the already proved bound \eqref{eq:affine-lower-long-horizon} applies with \(r\) in place of \(s\).
If \(r=s\), the rates coincide.
If \(r<s\), then \(T<2^{r+3}\) and \(n\le T\), whence
\[
\frac{R_r(T,n)}{R_s(T,n)}
=\left(\frac nT\right)^{1/(r+2)-1/(s+2)}
\ge T^{-1/(r+2)}
>2^{-(r+3)/(r+2)}\ge\frac14.
\]
The instance has affine dimension at most \(r\le s\) and the same universal group-size bound.

For \(2\le n\le T<8\), use the categorical instance in \eqref{eq:unified-lower-endpoint}.
For every allowed \(s\),
\[
\frac{\sqrt{nT}}{R_s(T,n)}
=\left(\frac nT\right)^{s/[2(s+2)]}
\ge\sqrt{\frac nT}
\ge\sqrt{\frac27}>\frac12.
\]
Thus all cases satisfy
\[
\E[\MCerr_{\calG}^{(1)}(T)]
\ge\frac b{32768}\frac{R_s(T,n)}{L_T^2}.
\]
Indeed \(b\le1/16\) and \(L_T>1\) make this constant smaller than the endpoint constants.
In every case the contexts, groups, and product outcome law are fixed before choosing the forecaster.
Taking \(c=b/32768\) and \(C=1+64C_{\rm sub}\) proves all claims of \Cref{thm:affine-lower}.

\par\smallskip
\noindent\phantomsection\label{proof:lower-step3b}\textbf{Details of Step 3.}
To verify the hypotheses of \citet[Theorem~4]{collina2026multicalibration} for the pair subsequences, generate the pair indices independently, with \(\Pr(K_t=j)=2/n\) for \(j\in[k]\) and \(\Pr(K_t=0)=(n-2k)/n\).
Independently generate fair sign arrays \((\xi_{j,r})_{r\le T}\) for every pair \(j\).
At its \(r\)-th visit, choose \(J_t=2j-1\) for sign \(+1\) and \(J_t=2j\) for sign \(-1\).
When \(K_t=0\), set \(J_t=n\).
Given the pair indices, all used sign entries are distinct and independent, so this generates exactly the independent uniform distribution of \(J_{1:T}\).
The noise increments defined in Step~3 satisfy \(Z_{j,r}=-\xi_{j,r}/2\), and \(N_j=\sum_vn_{v,j}\).
The exact pair projection is
\[
U_{v,2j-1}-U_{v,2j}
=\sum_{r=1}^{N_j}\one\{p_{\tau_{j,r}}=v\}Z_{j,r}
=-\frac12\sum_{r=1}^{N_j}\one\{p_{\tau_{j,r}}=v\}\xi_{j,r}.
\]
Every draw outside the pair contributes zero to this difference.

Label each prediction value by the round on which it first appears.
Fix \(j\) and condition only on an event \(N_j=d\) of positive probability.
The learner's seed, the full pair-index sequence, and the other pairs' sign arrays form auxiliary information independent of the entire \(j\)-th sign array, also under this conditional law.
Before \(\tau_{j,r}\), every observation is determined by that auxiliary information and \(\xi_{j,1},\ldots,\xi_{j,r-1}\).
The prediction label at \(\tau_{j,r}\) is therefore measurable before the current target sign.
For \(d\ge2\), the cited theorem with amplitude \(1/2\) and \(b\le C_{\rm rev}\) gives
\[
\E\!\left[\sum_v|U_{v,2j-1}-U_{v,2j}|\ \middle|\ N_j=d\right]
\ge\frac{b}{2\log(d+1)}
\E\!\left[\sum_v\sqrt{n_{v,j}}\ \middle|\ N_j=d\right].
\]
For \(d=0\), both sides vanish.
For \(d=1\), the absolute projection sum is \(1/2\), and the sum of square roots is one.
Thus the same bound with denominator \(L_T\) holds because \(b\le1/16<\log2\le L_T\).
For \(d\ge2\), use \(\log(d+1)\le L_T\).
Averaging over \(N_j\), summing over pairs, and using
\(\sum_j|U_{v,2j-1}-U_{v,2j}|\le\|U_v\|_1\) yields
\[
\E\!\left[\sum_v\|U_v\|_1\right]\ge\frac{b}{2L_T}\E[D],
\]
as used in Step~3.

\par\smallskip
\noindent\phantomsection\label{proof:lower-step5}\textbf{Details of Step 5.}
Write \(C_T=\MCerr_{\{\one\}}^{(1)}(T)\), \(A_T=\sum_t\|p_t-q_0\|_1\), and
\(\phi(v)_i=\operatorname{sign}(v_i-q_{0,i})\).
The predictable signs and the constant group give
\[
\begin{aligned}
\E[A_T]
&=\E\!\left[\sum_t\langle\phi(p_t),p_t-y_t\rangle\right]\\
&=\E\!\left[\sum_v\left\langle\phi(v),\sum_{t:p_t=v}(p_t-y_t)\right\rangle\right]
\le\E[C_T].
\end{aligned}
\]
Again the conditional expectation is taken roundwise before grouping.
For \(S_T=\sum_t(y_t-q_0)\), the triangle inequality gives
\[
\|S_T\|_1\le C_T+A_T,\qquad
\E\|S_T\|_1\le2\E[C_T].
\]
Fix a coordinate and write \(X_t=\one\{J_t=i\}-1/n\), with variance
\(\sigma^2=(1/n)(1-1/n)\).
Its sum \(Z=\sum_tX_t\) has variance \(v=T\sigma^2\ge1/2\).
The identity \(\E X_t^4=\sigma^2(1-3\sigma^2)\) and independence give
\[
\E Z^4=T\sigma^2(1-3\sigma^2)+6\binom T2\sigma^4
=3v^2+v(1-6\sigma^2)\le5v^2.
\]
Only terms with no singly occurring index survive the fourth-moment expansion.
H\"older's inequality yields
\[
v\le(\E|Z|)^{2/3}(\E|Z|^4)^{1/3},\qquad
\E|Z|\ge\frac{v^{3/2}}{(\E|Z|^4)^{1/2}}
\ge\sqrt{v/5}\ge\sqrt{T/(10n)}.
\]
Summing this last bound over coordinates gives
\(2\E[C_T]\ge\E\|S_T\|_1\ge\sqrt{nT/10}\), proving \eqref{eq:unified-lower-endpoint}.
\end{proof}

\section{General global and local packing guarantees}

The same algorithm admits guarantees for means near convex sets with controlled Hellinger packing numbers, even without an affine-dimension bound.

\subsection{Packing profiles and approximating families}

We use the same dyadic scales.
The logarithmic parameters \(L,\kappa\), the smoothing weight \(\beta\), the uniform prediction \(q_0=(1/n,\ldots,1/n)\), and \(J_T=K_{\max}\) are fixed explicitly in Appendix~\ref{app:universal-constants}.
The structural bound in \Cref{cor:discovery-residual} uses the actual counts \(M_I,B_I\).
We now bound those counts using global and local Hellinger packing.

We use the global and local packing numbers from \Cref{def:l1-packing}.

\begin{definition}[Approximating family]\label{def:approximating-family}
For a comparison depth \(2\le K\le K_{\max}\), an approximating family consists of nonempty closed convex sets \(S_1,\ldots,S_{K-1}\subseteq\DeltaD\) such that
\[
\max_{t\le T}\dist_H((1-\beta)\mu_t+\beta q_0,S_j)\le w_j/16\quad(1\le j<K).
\]
No nesting or assumption at finer scales \(j\ge K\) is required.
For such a family write
\[
\begin{aligned}
G_K&=1+\sum_{j=1}^{K-1}\GlobPack_H(S_j;w_j),\\
\Lambda_K(R)&=\max\left\{1,\max_{\substack{1\le j<K\\0\le h\le\lceil\log_2(2/w_j)\rceil}}
\LocPack_H(S_j;2^hw_j,R)\right\}.
\end{aligned}
\]
\end{definition}

For a fixed family through depth \(K\), write \(\Lambda=\Lambda_K(300\sqrt\kappa)\).
The sum in \(G_K\) bounds the number of candidate predictions across levels.
The maximum in \(\Lambda_K\) includes larger resolutions because boundary comparators are counted in distance annuli around their source.
These geometric quantities are used only in the analysis and are not computed by the algorithm.

\paragraph{A one-dimensional example.}
Let \(S=\Delta_1\), and write its points as \(p(\theta)=(\cos^2\theta,\sin^2\theta)\), \(0\le\theta\le\pi/2\).
Then \(H(p(\theta),p(\phi))=2\sin(|\theta-\phi|/2)\), which lies between \(2|\theta-\phi|/\pi\) and \(|\theta-\phi|\).
Consequently the global packing number is of order \(1+v^{-1}\) for \(0<v\le1\).
Inside a Hellinger ball of radius \(Rv\), any two points are at distance at most \(2Rv\), so their angles lie in an interval of length at most \(\pi Rv\).
Separation greater than \(v/2\) therefore gives at most \(1+2\pi R\) points.
Thus global packing grows at finer resolution, while local packing stays bounded across scales for fixed \(R\).

\subsection{General theorem and profile corollaries}

\begin{theorem}[Global and local packing bounds]\label{thm:general-l1-oracle}
On one event of probability at least \(1-\delta\), simultaneously for every comparison depth and approximating family in \Cref{def:approximating-family}, with \(w=w_{K-1}\),
\[
\begin{aligned}
\MCerr_{\calG}^{(1)}(T)\le{}&24\cdot10^6\sqrt\kappa\left[Tw+\min\left\{
\begin{aligned}
&G_K\sqrt{nL T},\\
&\sqrt{nL G_KT}+\frac{nLK_{\max}G_K\Lambda_K(300\sqrt\kappa)}w
\end{aligned}\right\}\right]\\
&+2K_{\max}+5.
\end{aligned}
\]
The deterministic bound \(\MCerr_{\calG}^{(1)}(T)\le2T\) also holds.
All packing quantities are used only in the analysis.
\end{theorem}

For the following corollaries, assume that the sets \(S_j\) satisfy Definition~\ref{def:approximating-family} at every positive scale used by the algorithm. All packing numbers are measured in Hellinger distance. The notation \(\widetilde O\) hides at most a factor of \(C[\log(en|\barG|(T+1)/\delta)]^2\), where \(C\) is universal.

\begin{corollary}[Global packing profile]\label{cor:h-global}
If \(\GlobPack_H(S_j;w_j)\le A w_j^{-s}\) for all \(j\), where \(A\ge1\) and \(s\ge0\), then
\[
\MCerr_{\calG}^{(1)}(T)\le\widetilde O\left(A^{1/(s+1)}n^{1/[2(s+1)]}T^{(2s+1)/(2s+2)}\right).
\]
\end{corollary}
\begin{corollary}[Global and local packing profiles]
Suppose the preceding global condition holds and, for every \(2\le K\le K_{\max}\), every local profile used in \(\Lambda_K(300\sqrt\kappa)\) is at most \(B n^{q/2}\), where \(B\ge1\) and \(q\ge0\).
Then
\[
\MCerr_{\calG}^{(1)}(T)\le\widetilde O\left((AB)^{1/(s+2)}n^{(1+q/2)/(s+2)}T^{(s+1)/(s+2)}\right).
\]
\end{corollary}
\begin{corollary}[Polynomial local packing growth]\label{cor:h-polynomial}
Suppose \(\LocPack_H(S_j;v,R)\le H_0R^s\) for every required \(v=2^h w_j\) and every \(R\ge1\), where \(H_0\ge1\) and \(s\ge0\).
Then
\[
\MCerr_{\calG}^{(1)}(T)\le\widetilde O\left(H_0^{2/(s+2)}n^{1/(s+2)}T^{(s+1)/(s+2)}\right).
\]
\end{corollary}

\subsection{General packing-count lemma}

The following lemma converts global and local packing bounds into discovery counts.

\begin{lemma}[Global and chronological annular packing]\label{lem:annular-packing}
On the event of \Cref{lem:common-event}, for every approximating family through depth \(K\) and every interval \(I\) at this depth defined in \Cref{def:discovery-counts},
\[
M_I\le G_K,
\qquad
B_I\le\min\{G_K,\ J_T\Lambda_K(300\sqrt{\kappa})\}.
\]
\end{lemma}
\begin{proof}
Work on \(\mathcal E_{\rm global}\) and fix an approximating family through depth \(K\) and an interval \(I\) at this depth.
Throughout this proof, \(P_j\) abbreviates \(P_j(\max I+1)\) for \(j\ge1\).
We follow the proof of \Cref{lem:iclr-affine-counts}, using the sets \(S_j\) in place of \(S_\beta\) and their global and local packing numbers in place of the affine packing bounds.

First, for each \(1\le j<K\), every candidate prediction in \(P_j\) lies within distance \(w_j/8\) of \(S_j\).
Indeed, it is the mean \(z_E\) of a completed level-\(\ell\) block with \(\ell\ge j\).
Each smoothed conditional mean is within \(w_j/16\) of \(S_j\), so convexity of \(S_j\) and joint convexity of \(H^2\) give the same bound for their weighted average \(\bar\mu'_E\).
Applying \eqref{eq:completed-block-concentration} at level \(\ell\) yields
\[
\dist_H(z_E,S_j)
\le H(z_E,\bar\mu'_E)+\dist_H(\bar\mu'_E,S_j)
\le\frac{w_\ell+w_j}{16}
\le\frac{w_j}{8}.
\]
This argument uses only \(S_j\), not nesting of the approximating sets, and includes values inherited from earlier depths or inserted on the final round of \(I\).

\textbf{Step 1. Count candidate predictions \(M_I\).}
For each \(a\in P_j\), choose \(\tilde a\in S_j\) with \(H(a,\tilde a)\le w_j/8\).
The insertion rule preserves separation after this replacement:
\[
H(\tilde a,\tilde b)
\ge H(a,b)-H(a,\tilde a)-H(b,\tilde b)
>\frac{3w_j}{4}>\frac{w_j}{2}
\qquad(a\ne b).
\]
Hence \(|P_j|\le\GlobPack_H(S_j;w_j)\), and summing over levels gives
\[
M_I=\left|P_0\cup\bigcup_{j=1}^{K-1}P_j\right|
\le1+\sum_{j=1}^{K-1}|P_j|
\le1+\sum_{j=1}^{K-1}\GlobPack_H(S_j;w_j)=G_K.
\]

\textbf{Step 2. Count boundary comparators in chronological order.}
Fix a group, source \(a\), and positive level \(j<K\), and list the completed boundary blocks' comparators \(q_1,q_2,\ldots\in P_j\) in closure order.
Each earlier comparator is available throughout every later block, so it cannot be chosen again for a boundary block.
Thus the comparators are distinct, and their number is at most \(|P_j|\le G_K\).

To obtain the local packing bound, reuse the chronological comparison from Step 2 of the proof of \Cref{lem:iclr-affine-counts}.
That comparison uses only \eqref{eq:iclr-local-hellinger}, \(H(q_m,z_m)\le w_j\), and \(N_m\ge m_j=1024nL/w_j^2\), where \(z_m,N_m\) are the mean and weight of the block associated with \(q_m\).
It does not require an affine approximating set and gives
\[
H(q_m,q_i)\ge\frac{H(q_m,a)}{A}-w_j,
\qquad i<m,\qquad A=16\sqrt\kappa.
\]
This strengthens the separation for comparators far from \(a\).
Choose representatives \(\tilde q_i\in S_j\) within distance \(w_j/8\) of \(q_i\), and use the same near and far regions as in that proof.
\begin{itemize}
\item \textbf{Near region: \(H(q_i,a)\le4Aw_j\).}
Step 1 gives separation greater than \(w_j/2\) between representatives, and
\[
H(\tilde q_i,a)\le\frac{w_j}{8}+4Aw_j\le5Aw_j.
\]

\item \textbf{Far regions: \(4A2^hw_j<H(q_i,a)\le8A2^hw_j\), \(h=0,1,\ldots\).}
For two comparators in the same region, order them as \(i<m\).
The chronological comparison and the replacement error give
\[
\begin{aligned}
H(\tilde q_m,\tilde q_i)
&>(4\cdot2^h-1)w_j-\frac{w_j}{4}
>2\cdot2^hw_j,\\
H(\tilde q_i,a)&\le\frac{w_j}{8}+8A2^hw_j\le9A2^hw_j.
\end{aligned}
\]
\end{itemize}
In each region the representatives are more than \(v/2\) apart and lie within distance \(300\sqrt\kappa\,v\) of \(a\), with \(v=w_j\) in the near region and \(v=2^hw_j\) in the \(h\)-th far region, since \(9A<300\sqrt\kappa\).
The definition of \(\Lambda_K\) therefore bounds each region by \(\Lambda_K(300\sqrt\kappa)\).
If \(h\ge j\), the far-region condition would require
\[
H(q_i,a)>4A2^hw_j=8A2^{h-j}\ge8A>\sqrt2,
\]
which is impossible.
Thus only \(h=0,\ldots,j-1\) can contribute far regions, all among the resolutions defining \(\Lambda_K\).
Including the near region gives at most \(j+1\le K\le K_{\max}=J_T\) regions.
Combining this count with the distinct-comparator bound and taking the maximum over groups, sources, and positive levels yields
\[
B_I\le\min\{G_K,\ J_T\Lambda_K(300\sqrt\kappa)\}.
\]
Both bounds are at least one, so they also cover the minimum value of \(1\) in the definition of \(B_I\).
\end{proof}

\subsection{Proof of Theorem~\ref{thm:general-l1-oracle}}

\begin{proof}
Work on \(\mathcal E_{\rm global}\).
We first derive two interval bounds, then choose a comparison width, show that the algorithm cannot move to a coarser scale, and sum the interval errors.

\textbf{Step 1. Bound the error on each interval.}
For an interval \(I\) at depth \(K\ge2\), write \(w=w_{K-1}\).
Start from \Cref{lem:piece-count-residual} and substitute the counts from \Cref{lem:residual-charging}.

\emph{(i) Using only the prediction count.}
For a dyadic set of positive \(x\), splitting at \(\sqrt{a/b}\) and summing each geometric tail gives \(\sum_x\min(a/x,bx)\le4\sqrt{ab}\).
Using \(Q_j(I)\le\min(8M_I^2,|I|/m_j)\) and \(m_j=1024nL/w_j^2\) therefore bounds the lower-piece term by
\[
48\sqrt{8\cdot1024}\sqrt{\kappa}\,M_I\sqrt{nL|I|}
<4400\sqrt{\kappa}\,M_I\sqrt{nL|I|}.
\]
The unfinished-zero term is at most \(32\sqrt{nLM_I|I|}\), since \(m_0=256nL\).
Using \(M_I,\kappa\ge1\) combines these two terms with coefficient \(4400+32=4432\), while leaving the sampling term separate:
\[
\MCerr_{\calG}^{(1)}(I)\le6\sqrt\kappa\,|I|w
+4432\sqrt\kappa\,M_I\sqrt{nL|I|}
+4\sqrt{nLM_IT}+2\beta|I|.
\]

\emph{(ii) Combining with the boundary-block count.}
The proof of \Cref{cor:discovery-residual} in Appendix~\ref{subsec:general-packing-proof} gives the stronger estimate
\[
\MCerr_{\calG}^{(1)}(I)\le6\sqrt\kappa\,|I|w
+98816\sqrt\kappa\,\frac{nLM_IB_I}{w}
+4\sqrt{nLM_IT}+2\beta|I|.
\]
Both estimates hold on the same event for every permitted endpoint, so taking their minimum also gives the stronger two-branch form
\[
\MCerr_{\calG}^{(1)}(I)\le10^6\sqrt\kappa\left[|I|w+\min\left\{M_I\sqrt{nL|I|},
\frac{nLM_IB_I}{w}\right\}\right]+4\sqrt{nLM_IT}+2\beta|I|.
\]

For the subsequent controller analysis, the same estimates also imply the horizon-based form
\[
\MCerr_{\calG}^{(1)}(I)\le10^6\sqrt\kappa\left[Tw+\min\left\{M_I\sqrt{nLT},\
\frac{nLM_IB_I}{w}+\sqrt{nLM_IT}\right\}\right]+2\beta|I|.
\]
Indeed, \(|I|\le T\) and \(M_I,\kappa\ge1\) bound the first estimate's lower-piece and sampling terms by \(4436\sqrt\kappa\,M_I\sqrt{nLT}\).
For the second estimate, \(98816<10^6\) and \(4\le10^6\sqrt\kappa\) give the displayed second branch.

\textbf{Step 2. Choose a comparison width.}
For a fixed comparison family through depth \(K\), let \(F\) be either \(G_K\sqrt{nL T}\) or \(\sqrt{nL G_KT}+nL J_TG_K\Lambda_K(300\sqrt\kappa)/w\).
By \Cref{lem:annular-packing}, \(M_I\le G_K\) and \(B_I\le J_T\Lambda_K(300\sqrt\kappa)\) at that depth.
Substituting these counts into the horizon-based bound from Step 1 gives \(\MCerr_{\calG}^{(1)}(I)\le10^6\sqrt\kappa(Tw+F)+2\beta|I|\) for either choice of \(F\).
At every coarser depth the profile counts and inverse-width term cannot increase.
Let \(x=\max(w,F/T)\).
If \(x>1\), the deterministic bound \(2T\) proves the assertion.
Otherwise choose the smallest dyadic resolution \(v\in\{w_0,\ldots,w_{K_{\max}-1}\}\) with \(v\ge x\), so \(v<2x\).

\textbf{Step 3. The algorithm never moves to a coarser scale.}
At this depth every prefix obeys
\[
\MCerr_{\calG}^{(1)}(I)\le10^6\sqrt{\kappa}(Tv+F)+1/T\le2\cdot10^6\sqrt{\kappa}\,Tv+1/T<3\cdot10^6\sqrt{\kappa}\,Tv.
\]
The strict inequality follows from \(Tv\ge Tw_{K_{\max}-1}>1/2\).
Thus the algorithm cannot cross this depth, even though it never computes \(F\).
Only the sets at scales below this comparison depth were used.

\textbf{Step 4. Sum the interval errors.}
At an earlier failed endpoint the residual is at most its threshold plus two, since one sampled residual per group changes by norm at most two per round.
The endpoint belongs to the old interval and no new interval is created at \(T\).
Visited widths double, so their sum is less than twice the final width, which is at most \(v\).
The triangle inequality across intervals gives
\[
\MCerr_{\calG}^{(1)}(T)\le6\cdot10^6\sqrt{\kappa}\,Tv+2K_{\max}
\le24\cdot10^6\sqrt\kappa(Tw+F)+2J_T+1.
\]
The argument holds for both choices on the same event, proving the minimum.
Substituting \(J_T=K_{\max}\) and using \(2J_T+1\le2K_{\max}+5\) gives the stated form of the bound.
\end{proof}

\subsection{Proofs of the packing corollaries}

\begin{proof}[Proof of Corollaries~\ref{cor:h-global}--\ref{cor:h-polynomial}]
For the global case, \(G_K\le AJ_Tw^{-s}\).
Balance at \(w_*=[AJ_T\sqrt{nL/T}]^{1/(s+1)}\).
For the global--local case, the inverse term is at most \(ABn^{1+q/2}LJ_T^2w^{-(s+1)}\), so balance at \(w_*=[ABn^{1+q/2}LJ_T^2/T]^{1/(s+2)}\).
Its geometric mean with \(Tw\) dominates the sampling term because \(Bn^{q/2}J_T\ge1\).
For polynomial local growth, diameter \(\sqrt2\) gives \(\GlobPack_H\le H_0(\sqrt2/w_j)^s\), and \(\Lambda_K\le H_0(300\sqrt\kappa)^s\).
The inverse coefficient becomes \(H_0^2nLJ_T^2(300\sqrt{2\kappa})^s\), and the same balance applies.
In all cases \(w_*\ge T^{-1/2}\).
If \(w_*\le1\), round upward to a dyadic resolution less than \(2w_*\), which cannot enlarge an inverse-width term.
Otherwise use \(2T\le2Tw_*\).
Writing \(Z=\log(e n|\barG|(T+1)/\delta)\), the definitions give \(\kappa\le8Z\), \(L\le40Z\), and \(J_T\le6Z\).
The three logarithmic exponents are respectively \(1/2+3/[2(s+1)]\), \(1/2+3/(s+2)\), and \((s+4)/(s+2)\), all at most two.
The constant base in the third bound has exponent \(s/(s+2)<1\).
These substitutions give exactly the three displayed rates, and absorb \(2J_T+1\) with a universal constant.
They include \(s=0\) and \(T=1\).
\end{proof}

\section{Choice of universal constants}
\label{app:universal-constants}

Recall \(K_{\max}=2+\lceil\log_2T\rceil\), \(w_j=2^{1-j}\), \(q_0=(1/n,\ldots,1/n)\), and \(A_{\max}=1+|\barG|T\).
The algorithm's smoothing weight and coordinate floor are \(\beta=1/(2T^2)\) and \(\epsilon=\beta/n\).
Set
\[
\kappa=2+\log(n/\beta),\qquad
\eta=\frac{\log(n/\beta)}{8\kappa},\qquad J_T=K_{\max},
\]
and
\[
L=2+\log\!\left(\frac{8n|\barG|A_{\max}^2(T+1)^2(K_{\max}-1)}{\delta}\right),
\qquad m_j=\frac{1024nL}{w_j^2}.
\]
All these quantities are fixed before play.
For \(Z=\log(e n|\barG|(T+1)/\delta)\), their definitions give \(\kappa\le8Z\), \(L\le40Z\), and \(J_T\le6Z\).
Set \(C_{\rm test}=3\cdot10^6\), giving the controller threshold \(C_{\rm test}\sqrt{\kappa}\,T w_{K-1}\).

The coordinate floor gives \(0\le\ell(a,y')\le\log(1/\epsilon)=\log(n/\beta)\), as required by the normalized comparison updates.

\end{document}